\documentclass[11pt]{article}

\usepackage[utf8]{inputenc}

\usepackage[preprint]{acl}

\usepackage{times}
\usepackage{latexsym}

\usepackage[T1]{fontenc}
\usepackage[utf8]{inputenc}

\usepackage{microtype}

\usepackage{inconsolata}

\usepackage{graphicx}

\usepackage{siunitx}
\usepackage{url}
\usepackage{nicematrix}
\usepackage{mathtools}

\usepackage{amsmath}
\usepackage{amssymb}
\usepackage{mathtools}
\usepackage{amsthm}

\usepackage{xspace}
\usepackage{enumitem}
\usepackage{float}% If comment this, figure moves to Page 2

\usepackage[capitalize,noabbrev]{cleveref}

\usepackage{amsthm}
\newtheorem{definition}{Definition}
\newtheorem{thm}{Theorem}
\newtheorem{proposition}{Proposition}

\newtheorem{assumption}{Assumption}
\newtheorem{remark}{Remark}
\newtheorem{hypothesis}{Hypothesis}
\newtheorem{corollary}{Corollary}

\newenvironment{proofsketch}{%
  \proof}{\endproof}

\usepackage[textsize=tiny]{todonotes}

\usepackage{xparse}  % In your preamble
\newcommand{\R}{\mathbb{R}}
\newcommand{\Loss}{\mathcal{L}}
\newcommand{\E}[1]{\mathbb{E}\left[#1\right]}
\NewDocumentCommand{\W}{m o}{{\mW^{(#1)}\IfValueTF{#2}{_{#2}}{}}}
\newcommand{\tp}{^\intercal}
\newcommand{\zero}{\mathbf{0}}
\newcommand{\mat}[4]{
    \begin{bNiceArray}{c|c}[margin]
      {#1} & {#2} \\
      \hline
      {#3} & {#4}
    \end{bNiceArray}
}
\newcommand{\df}[2]{\frac{\partial #1}{\partial #2}}
\newcommand{\defn}[1]{\textbf{#1}}

\newcommand{\colourbase}{black}
\newcommand{\mymacro}[2]{\newcommand{#1}{{\color{\colourbase}#2}}}
\usepackage{amsmath,amsfonts,bm}

\def\eqref#1{equation~\ref{#1}}
\def\1{\bm{1}}

\def\valpha{{\bm{\alpha}}}
\def\vbeta{{\bm{\beta}}}
\def\vgamma{{\bm{\gamma}}}
\def\va{{\bm{a}}}
\def\vb{{\bm{b}}}

\def\vp{{\bm{p}}}

\def\vt{{\bm{t}}}

\def\vy{{\bm{y}}}

\def\mE{{\bm{E}}}

\def\mH{{\bm{H}}}
\def\mI{{\bm{I}}}

\def\mM{{\bm{M}}}

\def\mP{{\bm{P}}}

\def\mS{{\bm{S}}}
\def\mT{{\bm{T}}}

\def\mW{{\bm{W}}}
\def\mX{{\bm{X}}}
\def\mY{{\bm{Y}}}

\DeclareMathAlphabet{\mathsfit}{\encodingdefault}{\sfdefault}{m}{sl}
\SetMathAlphabet{\mathsfit}{bold}{\encodingdefault}{\sfdefault}{bx}{n}

\newcommand{\softmax}{\mathrm{softmax}}

\title{On the Emergence of Induction Heads for In-Context Learning}

\author{Tiberiu Musat\thanks{Corresponding author. Email: \texttt{tmusat@ethz.ch}.}\>\qquad Tiago Pimentel\qquad Lorenzo Noci\qquad Alessandro Stolfo \\
\textbf{Mrinmaya Sachan \qquad Thomas Hofmann}  \\
ETH Zürich }

\begin{document}
\maketitle
\begin{abstract}
  Transformers have become the dominant architecture for natural language processing. Part of their success is owed to a remarkable capability known as \textit{in-context learning} (ICL): they can acquire and apply novel associations solely from their input context, without any updates to their weights. In this work, we study the emergence of \textit{induction heads}, a previously identified mechanism in two-layer transformers that is particularly important for in-context learning. We uncover a relatively simple and interpretable structure of the weight matrices implementing the induction head. We theoretically explain the origin of this structure using a minimal ICL task formulation and a modified transformer architecture. We give a formal proof that the training dynamics remain constrained to a 19-dimensional subspace of the parameter space. Empirically, we validate this constraint while observing that only 3 dimensions account for the emergence of an induction head. By further studying the training dynamics inside this 3-dimensional subspace, we find that the time until the emergence of an induction head follows a tight asymptotic bound that is quadratic in the input context length.
\end{abstract}

\section{Introduction}

% \Tiago{
% * High-level Motivation: Test-time training is important because of x. etc..
% * Expose the Problem, related to literature (what is missing in the literature; hint to what we're gonna solve)
% * Theoretical contributions: "In this paper, we formalise X. We then propose a principled way of measuring it with %Y."
% * Empirical contributions "Empirically, we experiment with X and measure ..."
% }

% How does intelligence emerge from \textit{gradient descent}? 
%
% Large language models (LLMs) have achieved highly advanced reasoning abilities, yet we still lack a principled account of how complex reasoning behaviors emerge from this simple learning rule. 
%
% Understanding the inner workings of LLMs is an important avenue towards developing novel AI systems with increased reliability and efficiency. 
%

LLMs possess a remarkable ability known as \textit{in-context learning} (ICL). A well-trained language model can learn and apply novel associations from their input context, without additional parameter updates \citep{brown2020language}. % This is in stark contrast to traditional \textit{in-weights learning}, where novel associations are directly encoded into the model weights.
% \tiago{We start with "How does intelligence emerge from \textit{gradient descent}?" And here say IWL is related to gradient descent, but ICL isn't. I think this can be a bit confusing.  \lore{ I replaced the prev sentence with: A well-trained LLM can learn and apply novel associations solely from their input context, without additional parameter updates. I want to say that further gradient updates are not needed. But still, the capability emerge after you have trained of course} }
%
%

Previous work by \citet{olsson2022context} traces back transformers' ICL capabilities to a learned mechanism termed \textit{induction head}: a pair of attention heads that implement a simple but powerful copying rule $\left[\; \ldots, A, B, \ldots, A \;\right] \to B$. Empirical work has shown that the emergence of induction heads co-occurs with a sharp decrease in the training loss and an increase in ICL accuracy \citep{olsson2022context, reddy2023mechanistic}.

\citet{reddy2023mechanistic} has proposed a three-parameter model of the induction head that can successfully account for the phenomena of abrupt learning and data distributional effects. However, the three parameters are inspired by the attention patterns of an induction head without a clear connection to the learned weights of a transformer. On the other hand, existing theoretical works that study the evolution of the weights during training are limited to simplified staged learning algorithms \citep{nichani2024transformers, bietti2024birth}.
%Furthermore, \citet{reddy2023mechanistic} has shown that induction heads can account for the effects of data distributional properties on the ICL--IWL trade-off \citep{chan2022data}.
%
%While a number of theoretical studies have established the emergence of induction heads using specific staged learning algorithms \citep{nichani2024transformers, bietti2024birth}, the precise dynamics during standard training remain elusive.

% This motivates the question of the current study:
% \begin{center}
% \emph{
%     How do induction heads emerge during training?
% }
% \end{center}
%
Our work aims to illuminate the precise training dynamics during the emergence of induction heads. To answer this question, we study the training dynamics of an autoregressive two-layer transformer using a minimal in-context learning task and a simplified architecture. We uncover a precise description of the model weights during training and a tight bound on the learning time.

%\newpage % These should go together:
Concretely, our contributions are as follows:
\begin{enumerate}[topsep=0.5ex,itemsep=0ex,partopsep=1ex,parsep=1ex]
    \item Using a minimal ICL formulation, we give a formal proof that training dynamics induce a simplified structure of the weights (\cref{sec:full_dynamics}). The evolution of model weights stays within a 19-dimensional subspace of the entire parameter space, regardless of model or task size. We index this subspace by introducing \textbf{19 pseudo-parameters}.
    \item We empirically find that only \textbf{3 pseudo-parameters} are learned at the end of training, corresponding exactly to an induction head (\cref{sec:empirical}). We also find that the emergence of the 3 parameters is \textit{self-contained}, unaided by the presence of the other 16 parameters.
    \item We theoretically study the training dynamics of the induction head, assuming that only the 3 parameters are learnable (\cref{sec:induction_head_dynamics}). We prove that the 3 parameters always emerge in a specific sequence. We also prove asymptotic bounds for the emergence time for each parameter in terms of the context length and a \textbf{tight bound on the total emergence time}.
    \item We train and interpret a standard attention-only transformer on an ICL task (\cref{sec:standard-transformers}). We find a relatively simple \textbf{description of the weight matrices} that implement the induction head. Using progress measures during training, we show that the emergence of induction heads in standard transformers closely resembles the minimal model.
\end{enumerate}
Finally, we also provide empirical validation for our theoretical results.

\section{Transformers and Induction Heads}
\label{sec:induction_heads}

An essential feature of natural language is the presence of \emph{reoccurring associations}. These can be as simple as repeated groups of words, such as the full names of people (e.g. ``Alan Turing'') or even the repeated occurrence of the phrase ``induction heads'' in this very document. However, we also observe the reoccurrence of very abstract associations, a phenomenon related to \emph{in-context learning}. Specific examples include fact retrieval, translation, and pattern matching \citep{olsson2022context}.

% \subsection{Language}

\mymacro{\token}{v}
\mymacro{\tokena}{a}
\mymacro{\tokenb}{b}
\mymacro{\tokens}{\mathbf{v}}
\mymacro{\vocab}{\mathcal{V}}
\mymacro{\btheta}{\boldsymbol{\theta}}
\mymacro{\ptheta}{p_{\btheta}}
\newcommand{\writemore}{\textcolor{red}{... (writemore)}\xspace}
\newcommand{\addcites}{\textcolor{red}{(addcites)}\xspace}
\mymacro{\embeddings}{\mathbf{X}}
% \mymacro{\embeddings}{\mathbf{X}}
\mymacro{\lookup}{\mathtt{LookUp}}
\mymacro{\hiddenstate}{\mathbf{h}}
\mymacro{\hiddenstates}{\mathbf{H}}
\mymacro{\selfattention}{\mathtt{SelfAttention}}
\mymacro{\layer}{\ell}
\mymacro{\layers}{L}
\mymacro{\unembedding}{\mathbf{W_{\texttt{unemb}}}}
\mymacro{\hiddensize}{d}
% \newcommand{\defn}[1]{\textbf{#1}}

% As mentioned above, induction heads are a specific structure that naturally arise in transformer-based language models.
% In this section, we formalise them.
Let $\tokens \in \vocab^*$ be a sequence of tokens in a language model's vocabulary $\vocab$.
In language modeling, our goal is typically to characterize the distribution $\ptheta(\token_t \mid \tokens_{<t})$.
Typically, this is implemented in a transformer via a series of operations.
First, we map a model's input tokens $\tokens_{<t}$ to embeddings $\embeddings \in \R^{\hiddensize \times t}$; this is done via a lookup operation:%\tiago{Add positions embeddings}
\begin{align}
    \embeddings = \lookup(\tokens_{<t})
\end{align}
Afterwards, the model implements a series of self-attention operations:%\tiago{MLPs?}
\begin{align}
    \hiddenstates^{\layer} = \selfattention(\hiddenstates^{\layer-1})
\end{align}
where, by definition, $\hiddenstates^{0} = \embeddings$.
Finally, in a last layer, the model predicts the next token as:
\begin{align}
    \ptheta(\token_t \mid \tokens_{<t}) = \softmax(\unembedding\,\hiddenstates^{\layers})
\end{align}
where $\unembedding \in \R^{|\vocab| \times \hiddensize}$ is known as an un-embedding matrix.

% \writemore\tiago{First name/last name or other examples?}
If a pair of tokens
% $\tokena\tokenb \in \vocab^2$\tiago{circ is concatenation} 
appears in the language model's context $\tokens_{<t}$ as $\token_{t'}\token_{t'+1} = \tokena\tokenb$, this pair is more likely than chance to appear again;
when we have $\token_{t-1} = \tokena$, we may thus want to 
% identify that it appeared in the context before, and 
predict $\tokenb$ as our next token in $\ptheta(\token_t \mid \tokens_{<t})$.

\defn{Induction heads} are a specific structure inside transformers that implements exactly this copying behavior. Given a prompt of the form $\ldots a b \ldots a$, an induction head predicts the token which follows the previous occurrence of $a$, in this case being $b$. Note that induction heads are not a modified type of attention head, but rather a mechanism learned by regular attention heads during standard training.

% In practice, an induction head is implemented in two steps.
We wish to, after observing $\token_{t-1} = \token_{t'}$, copy the value of $\token_{t'+1}$ into $\token_t$.
Due to the structure of self-attention, however, this requires at least two steps \citep{sanford2024one}.
First, the model must copy the token information from position $t'$ into $t'+1$.
Second, the newly obtained value enables the model to attend to position $t' + 1$ from $t-1$.

\begin{figure}[t!]
  \centering
  \captionsetup{width=.9\linewidth}
  \includegraphics[width=0.8\linewidth]{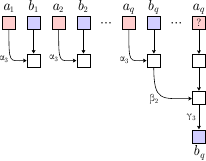}
  \caption{An induction head solving the \textit{in-context learning} (ICL) task. Given a series of item-label pairs, the model predicts the correct label for a query item. The first attention head retrieves the corresponding item for each label, enabling the second attention head to retrieve the correct label. Each path is modulated by one pseudo-parameter ($\alpha_3$, $\beta_2$, or $\gamma_3$).}
  \label{fig:induction-head}
\end{figure}

% More precisely, an induction head requires at least two attention layers.

% \Tiago{Old text below. I think it's a bit more smooth than above, so maybe can be merged.}

% Induction heads are attention heads that implement a simple but powerful algorithm. Given a prompt of the form $\left[\; \ldots, A, B, \ldots, A \;\right]$, an induction head predicts the token which follows the previous occurrence of $A$, in this case being $B$. Note that induction heads are not a modified type of attention head, but rather a mechanism learned by regular attention heads during standard training.

% Induction heads are composed of two attention layers. The first attention layer retrieves the value of $A$ into $B$ by attending to the previous token using positional embeddings. The newly obtained value enables the second attention layer to retrieve $B$ from the second occurrence of $A$. Note that two layers are necessary to solve the task since $B$ and the second $A$ initially have no shared information. 

\section{Induction Heads in Simplified Transformers}

\subsection{Disentangled Transformer}
\label{sec:setup}

In order to understand the emergence of induction heads, we study the training dynamics in a minimal formulation. We propose a simplified, but equally powerful, transformer architecture (depicted in \cref{fig:architecture}) with a \textit{disentangled} residual stream \citep{friedman2023learning}.

\begin{figure}[h]
  \centering
  \captionsetup{width=.9\linewidth}
  \includegraphics[width=0.8\linewidth]{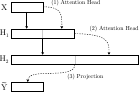} % Change letters
  \caption{Our minimal transformer architecture. We use two attention-only layers and a linear layer; we disentangle the attention layers by concatenating the inputs and outputs, rather than adding them together.}
  \label{fig:architecture}
\end{figure}

\subsubsection{Data Distribution}

We use a common ICL task formulation that requires labeling an item based on a list of $N$ item-label pairs \citep{chan2022data, reddy2023mechanistic, hochreiter2001learning}. The $i^{\text{th}}$ pair consists of an item $\va_i \in \R^D$ and a label $\vb_i \in \R^{D}$ with dimensionality $D \in \mathbb{N}$. We ask the model to predict the label for one of the items $\va_q$ where $q\in\{1,\ldots,N\}$.

We annotate each item with a positional embedding $\vp_i \in \R^{D}$ and each label with the rotated positional embedding $\mM \vp_i$, where $\mM \in \R^{D \times D}$. The rotation is fixed before training to create a learnable correlation similar to a sinusoidal embedding \citep{vaswani2017attention}. This enables the attention mechanism to connect the corresponding items and labels. We do not use any positional embedding for the query item. Assuming that $D$ is even, we use
\begin{equation}
    \mM = \begin{bNiceArray}{c|c}[margin]
        \zero_{(D/2) \times (D/2)} & \mI_{D/2}
        \\ \hline
        \mI_{D/2} & \zero_{(D/2) \times (D/2)} 
    \end{bNiceArray},
\end{equation}
where $\mI_{D/2} \in \R^{(D/2) \times (D/2)}$ is the identity matrix.

We concatenate items and labels with their positional embeddings to obtain our data:
\begin{align*}
&\mX_{2i-1, :} = \big[ \; \va_i^\intercal \; | \; \vp_i^\intercal \; \big]^\intercal \quad
\mX_{2i, :} = \big[ \; \vb_i^\intercal \; | \; \vp_i^\intercal \mM \; \big]^\intercal\\
&\forall i \in \{ 1, \ldots, N \} \\[0.5em]
&\mX_{2N+1, :} = \big[ \; \va_q^\intercal \; | \; \zero \; \big]^\intercal
\quad
\vy = \vb_q\\
&q\in \{ 1, \ldots, N \}
\end{align*}
where $\mX \in \R^{(2N+1) \times 2D}$, $\vy \in \R^{D}$, and $[\; \cdot \;|\; \cdot \;]$ denotes concatenation.

We assume a \textit{lexinvariant} language model \citep{huang2023lexinvariant} where items, labels, and positional embeddings are independent and identically distributed. For our theoretical results, we introduce additional assumptions on the distribution of items, labels, and positional embeddings, as needed.

We train our model with mean-squared error loss $\mathcal{L} = \| \, \vy - \tilde{\vy} \, \|^2$ using only the output of the query item located at the last position, i.e. $\tilde{\vy} = \widetilde{\mY}_{2N+1, :}$.

\subsubsection{Architecture}

Our transformer has two single-head attention layers followed by a linear layer. For the attention layers, we use a merged key-query matrix and no value/projection matrix. We disentangle the residual stream by directly concatenating the attention output to the existing residual stream:
\begin{align}
\mH_1 &= \Big[ \; \mX \; \Big| \; \sigma\Big(\mX \W{1} \mX\tp\Big) \, \mX \; \Big] , \\ 
\mH_2 &= \Big[ \; \mH_1 \;\Big| \; \sigma\Big(\mH_1 \W{2} \mH_1\tp\Big) \, \mH_1 \; \Big] , \\
\widetilde{\mY}\; &= \mH_2 \W{3} \, ,
\end{align}
where $[\; \cdot \;|\; \cdot \;]$ denotes matrix concatenation and $\sigma$ denotes the softmax function with autoregressive masking. $\mW^{(1)} \in \mathbb{R}^{2D \times 2D}$, $\mW^{(2)} \in \mathbb{R}^{4D \times 4D}$, $\mW^{(3)} \in \mathbb{R}^{8D \times D}$ are the learnable weights and $\mH_1 \in \mathbb{R}^{(2N + 1) \times 4D}$,  $\mH_2 \in \mathbb{R}^{(2N + 1) \times 8D}$, $\widetilde{\mY} \in \R^{(2N+1) \times D}$ denote the activations and final output.

\paragraph{Justification.}
Merged key-query matrices are commonly used in theoretical works \citep{edelman2024evolution, nichani2024transformers}. We do not use MLPs since they are neither necessary nor useful for the task at hand. The disentangled transformer has the same expressivity as a regular transformer with a large residual dimension.

\subsection{Training Dynamics}
\label{sec:full_dynamics}

Our model has a total of $28 D^2$ parameters. However, as we show below, the training dynamics on our data distribution remain constrained to a 19-dimensional subspace that we index using {19 pseudo-parameters}. Our theoretical result is based on the following assumptions:

\begin{assumption} \label{asspt:zero_initialisation}
    \defn{Zero Initialization.} We assume our neural network is initialized with zero weights, i.e. $\mW^{(1)} = \mathbf{0}$,
        $\mW^{(2)} = \mathbf{0}$, and
        $\mW^{(3)} = \mathbf{0}.$
\end{assumption}

The zero initialization is commonly used in theoretical works \citep{nichani2024transformers, edelman2024evolution}, being motivated as a reasonable approximation for small random initializations.

\begin{assumption} \label{asspt:population_loss}
    \defn{Population Loss.} We assume the network is trained with gradient descent over the entire data distribution at every step:
    \begin{equation}
        \mW^{(k)} \;\leftarrow \mW^{(k)}\; - \lambda \, \mathbb{E}\Biggl[ \frac{\partial \mathcal{L}}{\partial \mW^{(k)}} \Bigg] ,
    \end{equation}
    where $\lambda > 0$ is the learning rate.
\end{assumption}

\begin{assumption} \label{asspt:isotropy}
    \defn{Isotropic Data.} We assume that the data distribution is invariant to orthogonal transformations of items, labels, and positional embeddings:
    \begin{align}
        % \begin{gathered}
        &f\big(\{\va_i\}, \, \{\vb_i\}, \, \{\vp_i\}, \, q\big) \\
        &\qquad\quad= f\big(\{\mE \va_i\}, \, \{\mE \vb_i\}, \, \{\vp_i\}, \, q\big) \nonumber\\
        &\qquad\quad= f\big(\{\va_i\}, \, \{\vb_i\}, \, \{\mE \vp_i\}, \, q\big) \nonumber,
        % \end{gathered}
    \end{align}
    for any orthogonal matrix $\mE \in \mathbb{R}^{D \times D}$, where $f\big(\{\va_i\}, \, \{\vb_i\}, \, \{\vp_i\}, \, q\big)$ is the probability density over the items, labels, positional embeddings, and query index.
\end{assumption}

Note that this assumption is weaker than, for example, assuming a normal distribution, since a normal distribution is isotropic.

\begin{thm}
\label{thm:19-parameters}
    Assume that we train a disentangled transformer from zero initialization with population loss on isotropic data on our ICL task. 
    Then, the weight matrices will have the following structure throughout the entire training process:
        \begin{align*}
            &\mW^{(1)} \;=\;
            \begin{bNiceArray}{c|c}[margin]
              {\valpha_1 \mI} & \mathbf{0} \\
              \hline
              \mathbf{0} & {\valpha_2 \mI + \valpha_3 \mM}
            \end{bNiceArray},
            \\[0.5em]
            &\mW^{(2)} \;=\;\\
            &\begin{bNiceArray}{c|c|c|c}[margin]
              {\vbeta_1 \mI} & \mathbf{0} & {\vbeta_2 \mI} & \mathbf{0} \\
              \hline
              \mathbf{0} & {\vbeta_3 \mI + \vbeta_4 \mM} & \mathbf{0} & {\vbeta_5 \mI + \vbeta_6 \mM} \\
              \hline
              {\vbeta_7 \mI} & \mathbf{0} & {\vbeta_8 \mI} & \mathbf{0} \\
              \hline
              \mathbf{0} & {\vbeta_9 \mI + \vbeta_{10} \mM} & \mathbf{0} & {\vbeta_{11} \mI + \vbeta_{12} \mM}
            \end{bNiceArray},
            \\[0.5em]
            &\mW^{(3)} \;=\;\\
            &\quad\begin{bNiceArray}{c|c|c|c|c|c|c|c}[margin]
              {\vgamma_1 \mI} & \mathbf{0} &
              {\vgamma_2 \mI} & \mathbf{0} &
              {\vgamma_3 \mI} & \mathbf{0} &
              {\vgamma_4 \mI} & \mathbf{0}
            \end{bNiceArray}\tp,
        \end{align*}
    where we collect the parameters of each weight matrix in three vectors $\valpha \in \R^3$, $\vbeta \in \R^{12}$ and $\vgamma \in \R^4$ that vary throughout training.
\end{thm}

\begin{proofsketch}
    We give an inductive proof by showing that, if weights have the above structure, then their gradients also have the same structure. Since the zero initialization fits the structure, this ensures that the structure is preserved during training.

    To prove the structure of the gradient, we apply a carefully chosen rotation to the entire data distribution. Since the data distribution is isotropic, the rotation will not change the data distribution, so the expected gradient will also remain unchanged.
    
    However, we are also able to show that our rotation induces a specific similarity transformation of the gradient:
    \begin{equation}
        \mathbb{E}\Biggl[ \frac{\partial \mathcal{L}}{\partial \W{k}[ij]} \Bigg]
        \;=\;
        F \; \mathbb{E}\Biggl[ \frac{\partial \mathcal{L}}{\partial \W{k}[ij]} \Bigg] \; F^\intercal ,
    \end{equation}
    where $F$ is an orthogonal or block-orthogonal matrix and $\W{k}[ij]$ is a block of a weight matrix. From this, we are able to show that the expected gradient must have the desired structure. We give the full proof in \cref{app:weight_structure}.
\end{proofsketch}

\begin{figure}[h!]
  \centering
  \captionsetup{width=.9\linewidth}
  \includegraphics[width=\linewidth]{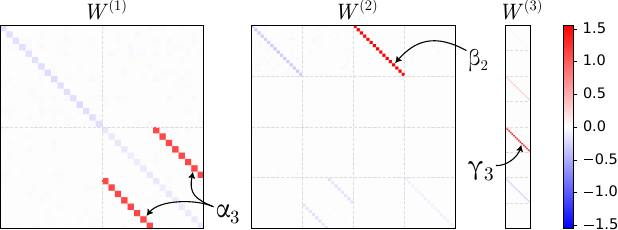}
  \caption{Weights at the end of standard training have the theoretically predicted structure.}
  \label{fig:weights}
\end{figure}

\paragraph{Empirical Validation.}
In \cref{fig:weights}, we confirm our theoretical result by visualizing the weights at the end of training with stochastic gradient descent. We provide the full training details and additional figures in \cref{app:model_weights}.

\subsection{Emergence of Induction Heads}
\label{sec:empirical}

We now proceed to studying the evolution of these 19 pseudo-parameters during training. By observing or ablating specific parameters, we are able to test two hypotheses regarding the emergence of induction heads.

\begin{hypothesis}[due to \citet{olsson2022context}] \defn{Induction Head Phase Transition.}
    Reaching low training loss on our ICL task coincides with the emergence of an induction head.
    % as defined in \cref{sec:induction_heads}.
\end{hypothesis}

We can already see from \cref{fig:weights} that three parameters have a larger magnitude, namely $\valpha_3$, $\vbeta_2$, and $\vgamma_3$. Interestingly, the mechanism performed by these three parameters together corresponds exactly to an induction head. In the first layer, $\valpha_3$ makes each label attend to the preceding item. In the second layer, $\vbeta_2$ makes the query item attend to the correct label based on the newly retrieved item. Finally, $\vgamma_3$ outputs the label retrieved by the second layer. In \cref{fig:training-3d-h1}, we visualize the 19 pseudo-parameters and loss during training, confirming that the drop in loss is driven by the emergence of the induction head.

\begin{hypothesis} \defn{Self-Contained Dynamics.}
    The emergence of the induction head is unaided by the presence of any other parameter.
\end{hypothesis}

By training the model while constraining its parameters to the 3-dimensional subspace spanned by the three parameters, we uncover very similar dynamics. As depicted in \cref{fig:training-3d-h2}, we find that the emergence of the induction head is unaffected, even slightly accelerated. We show a few more plots and full training details in \cref{app:dynamics}.

\begin{figure}[h!]
  \centering
  \captionsetup{width=.9\linewidth}
  \includegraphics[width=0.9\linewidth]{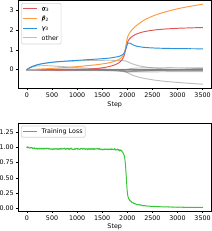}
  \caption{The value of the 19 pseudo-parameters during training \emph{(top)} and the associated training loss \emph{(bottom)}.}
  \label{fig:training-3d-h1}
\end{figure}

\begin{figure}[h!]
  \centering
  \captionsetup{width=.9\linewidth}
  \includegraphics[width=0.9\linewidth]{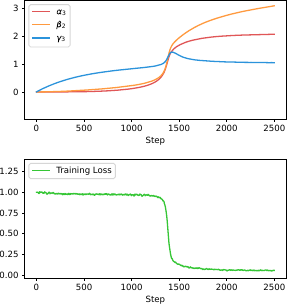}
  \caption{Ablating all parameters except $\valpha_3$, $\vbeta_2$, and $\vgamma_3$ results in strikingly similar dynamics.}
  \label{fig:training-3d-h2}
\end{figure}

\subsection{Training Dynamics of Induction Heads}
\label{sec:induction_head_dynamics}

Motivated by the empirical results in the previous section, we study the training dynamics constrained to the 3-dimensional subspace spanned by $\valpha_3$, $\vbeta_2$, and $\vgamma_3$, finding several tight bounds for the emergence of the induction head.
Specifically, we have the following assumptions:

% \subsubsection{Theoretical Results}
% \label{sec:induction_head_dynamics_theory}

% We study the emergence of an induction head under the following assumptions:

\begin{assumption} \label{asspt:three_param_model}
    \defn{Three-Parameter Model.} Only parameters $\valpha_3, \vbeta_2,$ and $\vgamma_3$ are learnable.
    %For the rest of the paper, we also refer to these parameters as simply $\alpha, \beta,$ and $\gamma$.
\end{assumption}

\begin{assumption} \label{asspt:gradient_flow}
    \defn{Gradient Flow.} We study the training dynamics under the assumption of a continuous-time gradient flow with unit learning rate,
    \begin{align*}
        \frac{\partial \alpha_3}{\partial t} = - \frac{\partial \mathcal{L}}{\partial \alpha_3} ,
        \quad
        \frac{\partial \beta_2}{\partial t} = - \frac{\partial \mathcal{L}}{\partial \beta_2} ,
        \quad
        \frac{\partial \gamma_3}{\partial t} = - \frac{\partial \mathcal{L}}{\partial \gamma_3} ,
    \end{align*}
    where $\alpha_3, \,\beta_2, \,\gamma_3 : \mathbb{R}_{\ge0} \to \mathbb{R}$ are the continuous-time trajectories of the three parameters.
\end{assumption}

\begin{assumption} \label{asspt:zero_initialisation}
    \defn{Zero Initialization.} We assume our neural network is initialized with zero weights. Equivalently, $\alpha_3(0) = \beta_2(0) = \gamma_3(0) = 0$.
\end{assumption}

\begin{assumption} \label{asspt:whitened_inputs}
    \defn{Orthonormal Inputs.} We assume that all items, labels, and positional embeddings are orthogonal and have unit norm. Specifically,
    \begin{gather}
        \| \va_i\| = \| \vb_i \| = \|\vp_i\| = 1, \\
        \va_i^\intercal \va_j = \vb_i^\intercal \vb_j = \va_i^\intercal \vb_i = 0, \\
        \vp_i^\intercal \vp_j = \vp_i^\intercal \mM  \vp_i = \vp_i^\intercal \mM  \vp_j = 0 ,
    \end{gather}
    for all $i,j \in \{ 1, 2, \ldots, N\}, i\ne j$.
\end{assumption}

%\begin{remark}
Note that this assumption requires $D \ge 2N$. There are two ways to motivate this assumption, either by preprocessing the inputs using a whitening transformation, or by considering a very large dimension $ D\rightarrow\infty$ and vectors sampled from an i.i.d. Gaussian with variance $1 / \sqrt{D}$.
%\end{remark}

\begin{assumption} \label{asspt:query_last}
    \defn{Query Last.} We assume that the query item always refers to the last item-label pair present in the sequence, or $q=N$.%\tiago{Were you not able to prove this with uniform q? Section 2 defines this as uniform, so at least say why you chose something else here.\tb{section 2 was referring only to experiments}}
\end{assumption}

Note that even if the target label’s position is fixed, a full induction head is still required: the model cannot directly attend to specific positions because positional embeddings are randomly generated and carry no explicit location information.

\begin{definition} \defn{Parameter Emergence Time.}
    We say that each of the parameters $\alpha$, $\beta$, or $\gamma$ has emerged when its value becomes greater than $1/2$ for the first time:
    \begin{gather}
        T_{\alpha} = \inf\Bigl\{t\Bigm|
        \alpha_3(t)\ge\tfrac12 \Bigr\} , \\
        T_{\beta} = \inf\Bigl\{t \Bigm|
        \beta_2(t)\ge\tfrac12 \Bigr\}  , \\
        T_{\gamma} = \inf\Bigl\{t \Bigm|
        \gamma_3(t)\ge\tfrac12\Bigr\} ,
    \end{gather}
where $t \in\mathbb{R}_{\ge0}$.
\end{definition}

\begin{thm}
\label{thm:tight-bound}
    Assume that inputs are orthonormal and that only parameters $\alpha_3, \beta_2,$ and $\gamma_3$ are learnable. In this case, for large enough $N$, parameters always emerge in the order $T_\gamma < T_\beta <T_\alpha$ and the emergence times asymptotically follow:
    \begin{align*}
        T_\alpha = \Theta\Big(N^2 \Big) , \;\;
        T_\beta = \Theta\Big(N^2 \Big) , \;\;
        T_\gamma = \Theta\big(N\big) ,
    \end{align*}
    with $N$ item-label pairs in the context.
\end{thm}
\begin{proof}[Proof Sketch]
    The proof is based on proving bounds for the gradient of each parameter.
    Before the emergence of any parameter, we have that $\partial \gamma_3 / \partial t = \Theta(1/N)$, while $\partial \alpha_3 / \partial t = O(1/N^2)$ and $\partial \beta_2 / \partial t = O(1/N^2)$. This implies that $\gamma_3$ emerges first in $\Theta(N)$. Afterwards, we show that $\partial \beta_2 / \partial t = \Theta(1/N^2)$ and $\partial \beta_2 / \partial t > \partial \alpha_3 /\partial t $. This implies that $\beta_2$ emerges next in $\Theta(N^2)$. Finally, we show that $\partial \alpha_3 / \partial t = o(1/N^2)$, which implies that $\alpha_3$ emerges last in $\Theta(N^2)$. See the full proof in \cref{app:tight_bound}.
\end{proof}

\begin{definition}
    \defn{Induction Head.} We say that an induction head has emerged if all three parameters are greater than $1/2$.
\end{definition}

\begin{definition}
    \defn{Time until ICL.} We say that in-context learning has emerged at the first time when the induction head is present. Specifically,%\tiago{Loss does not go to zero without a dependence on context length.\tb{What is meant here?}}
    \begin{gather*}
        t_{\mathrm{ICL}}
        =\inf\Bigl\{\,t \Bigm|         \min\big(\alpha_3(t),\;\beta_2(t),\;\gamma_3(t)\big)\ge\tfrac12\,
        %\alpha(t)\ge\tfrac12,\;\beta(t)\ge\tfrac12,\;\gamma(t)\ge\tfrac12\,
        %=\inf\Bigl\{\,t\in\mathbb{R}_{\ge0}\;\Bigm|\;         \alpha(t)\ge\tfrac12,\;\beta(t)\ge\tfrac12,\;\gamma(t)\ge\tfrac12\,
        \Bigr\} .
\end{gather*}

\end{definition}

\begin{corollary}
    The time until the emergence of in-context learning asymptotically follows:
    \begin{align*}
        t_{\text{ICL}} = \Theta\Big(N^2 \Big) ,
    \end{align*}
    where $N$ is the number of item-label pairs in the context.
\end{corollary}

\paragraph{Empirical Validation.}
We validate our theoretical results in \cref{fig:emergence}. Unlike our theory, we use a threshold of $0.1$ for $\valpha_3$ and $\vbeta_2,$ (rather than $0.5$) to better highlight their separation. Training details in \cref{app:training_details_ihd}.

\begin{figure}[h]
  \centering
  \captionsetup{width=.9\linewidth}
  \includegraphics[width=0.9\linewidth]{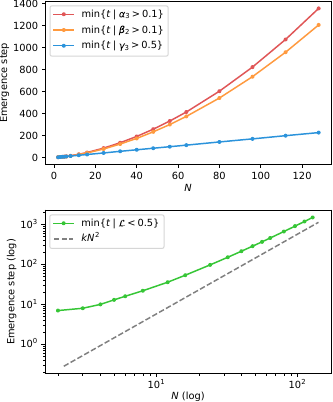}
  \caption{\emph{Top:} The time until the emergence of $\valpha_3$, $\vbeta_2,$ and $\vgamma_3$ for different values of $N$. \emph{Bottom:} Time until the emergence of in-context learning (log scale) and its quadratic asymptote.}
  \label{fig:emergence}
\end{figure}

\section{Induction Heads in Standard Transformers}
\label{sec:standard-transformers}

\subsection{Setup}

We validate the insights from the previous section and use them to understand standard transformers. We train an autoregressive transformer following the recipe of \citet{vaswani2017attention}. We train the model using synthetic data to predict the label of a query item based on the preceding item-label pairs, as depicted in \cref{fig:induction-head}. We use only two attention-only layers with one attention head per layer. We remove MLPs from the transformer blocks since they are neither necessary nor useful for the task at hand. We apply a cross-entropy loss only on the query item. We specify the full training details in \cref{app:transformer_training_details}.

\subsection{Weight Matrix Structure}

\paragraph{Notation.} We train a transformer with residual stream dimension $D$, vocabulary size $N_T$, and block size $N_P$. We denote the embedding of token $i$ as $\vt_i \in \R^D$ and the embedding of position $j$ as $\vp_j \in \R^D$. We deonte the token and positional embedding matrices as $\mT \in \R^{D \times N_T}$ and $\mP \in \R^{D \times N_P}$, respectively. The layer $l \in \{ 1,  2\}$ has the query, key, value, and output matrices $\mW_Q^l, \mW_K^l, \mW_V^l \in \R ^{D_H\times D}$, and $\mW_O^l \in \R^{D \times D_H}$, respectively. A final linear output layer $\mW_o \in \R^{N_T \times D}$ is applied. We use $N_T = N_P = 32$.

\paragraph{Architecture.} For an input sequence of length $L$ with tokens $\mX_T \in \R^{N_T \times L}$ and positions $\mX_P \in \R^{N_P \times L}$ encoded as one-hot vectors, the model will compute residual streams $\mH_l \in \R^{D \times L}$ as
\begin{align}
    &\mH_0 = \mT \mX_T + \mP \mX_P
    \\
    &\mH_{l} = \mH_{l-1} + \mW_O^{l} \mW_V^{l} \mH_{l-1} \, \sigma(\mS_{l})
    \\
    &\mS_l = (\mW_K^l \mH_{l-1})^\intercal (\mW_Q^l \mH_{l-1})
\end{align}
where $l \in \{1,2\}$ and  $\sigma$ denotes column-wise softmax with causal masking. $\mS_l \in \R^{L \times L}$ is the matrix of attention scores where $(\mS_l)_{i,j}$ denotes the attention score paid by position $i$ to position $j$. The output is computed using a linear layer and softmax $\mY=\sigma(\mW_o \mH_2) \in \R^{N_T \times L}$.

\paragraph{Explanation.} There are only 4 sub-spaces of the residual stream that are ever activated. First, there is the space spanned by the initial token and positional embeddings, $\vt_i$ and $\vp_j$. Second, there is the space where the first head writes the retrieved embeddings, $\mW^1_O \, \mW^1_V \, \vt_i$ and $\mW^1_O \, \mW^1_V \, \vp_j$. Third, there is the space where the second head writes the retrieved embeddings, $\mW^2_O \, \mW^2_V \, \vt_i$ and $\mW^2_O \, \mW^2_V \, \vp_j$. Finally, the second head could retrieve the output of the first head, creating a fourth subspace spanned by $\mW^2_O \, \mW^2_V \, \mW^1_O \, \mW^1_V \, \vt_i$ and $\mW^2_O \, \mW^2_V \, \mW^1_O \, \mW^1_V \, \vp_j$. Since there are $N_T$ tokens and $N_P$ positions, each of the four subspaces will have $N_T + N_P$ dimensions. Moreover, each subspace is highly interpretable, as it can be indexed directly by the corresponding token or position.

\paragraph{Visualization.} Using these intepretable directions, we can understand the mechanism performed by each layer. For example, $(\mW_K^1 \, \vp_i)^\intercal \, \mW_Q^1 \, \vp_j$ represents exactly the attention score paid by position $i$ to position $j$ during the first layer. In \cref{fig:transformer_interp}, we visualize the key-query matrix products and final output matrix, indexed by these highly interpretable dimensions. Note that this transformation shows how to compute attention scores and outputs for any possible input sequence, and hence \cref{fig:transformer_interp} is a complete description of the model behavior.

\subsection{Induction Head Mechanism}

In \cref{fig:transformer_interp}, we can see that our weights have a relatively simple and interpretable structure that mirrors the minimal model. Each layer is dominated by a diagonal or subdiagonal within a single block that perfectly aligns the corresponding subspaces. The first layer attends to the previous position. The second layer attends to the token retrieved by the first layer. The final layer outputs the token retrieved by the second layer. This clarifies the structure of the weight matrices that underlie the induction head mechanism.

\begin{figure}[h]
  \centering
  \includegraphics[width=\linewidth]{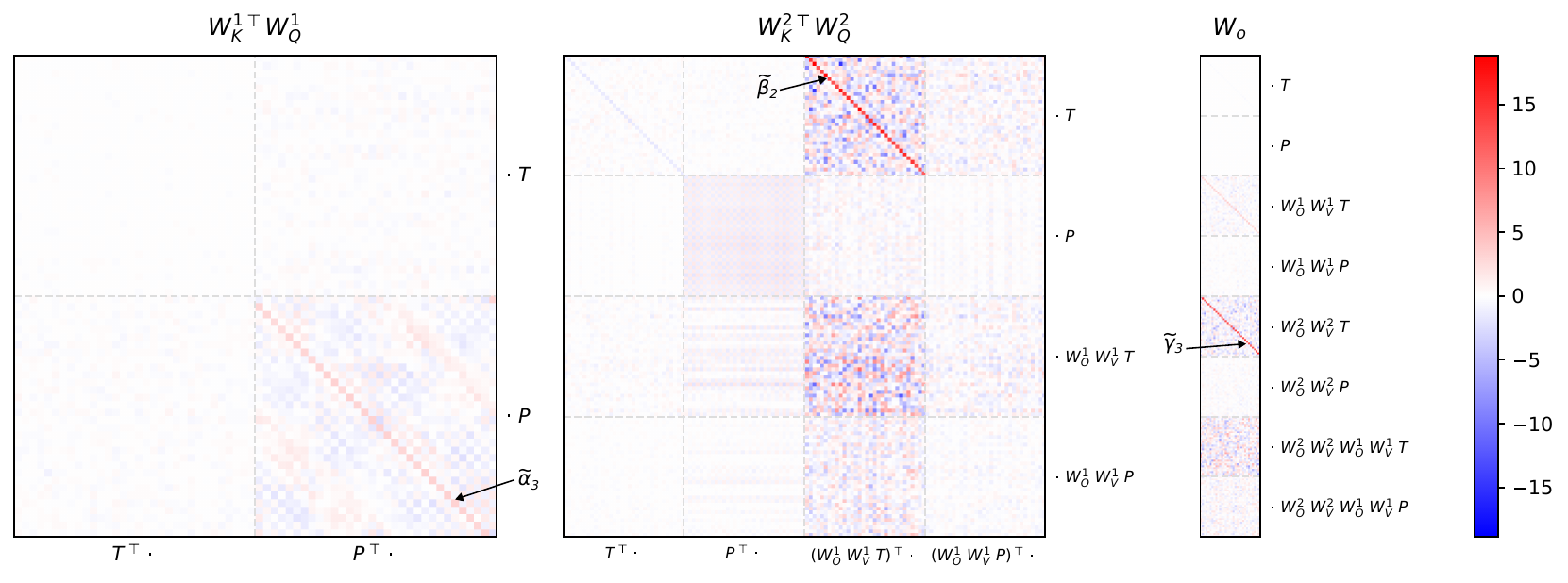}
  \caption{The weights of a two-layer attention-only transformer under our interpretable transformation. Dots $\cdot$ denote matrix multiplication. For example, the bottom-right block of the left plot, $\mP^\intercal \, {\mW_K^1}^\intercal \mW_Q^1 \mP$, is dominated by the subdiagonal, showing that each position attends to the previous. Some noise is due to random initialization and stochastic gradient descent.}
  \label{fig:transformer_interp}
\end{figure}

\subsection{Training Dynamics via Progress Measures}
\label{sec:progress-measures}

To better understand the emergence of induction heads in standard transformers, we devise three progress measures corresponding exactly to the three relevant structures identified in \cref{fig:transformer_interp}:
\begin{align}
    &\widetilde{\alpha}_3 = \sum_{i=1}^{N_P}  (\mW_K^1 \; \vp_{i-1})^\intercal (\mW_Q^1 \; \vp_{i})
    \\
    &\widetilde{\beta}_2 = \sum_{i=1}^{N_P}  (\mW_K^2 \; \mW_O^1 \; \mW_V^1 \; \vt_{i})^\intercal (\mW_Q^2 \; \vt_{i})
    \\
    &\widetilde{\gamma}_3 = \mathrm{tr}(\mW_o \; \mW_O^2 \; \mW_V^2 \; \mT )
\end{align}

We train a larger version of the previous transformer and plot the evolution of the progress measures in \cref{fig:progress_measures}. We observe the same order of emergence as in the minimal formulation: first $\widetilde{\gamma}_3$, then $\widetilde{\beta}_2$, and finally  $\widetilde{\alpha}_3$. However, the use of the cross-entropy loss introduces an important difference: achieving low loss requires much larger parameter magnitudes, particularly $\widetilde{\gamma}_3$ and $\widetilde{\beta}_2$. This results in different post-emergence dynamics and delayed loss reduction.

\begin{figure}[H]
  \centering
  \includegraphics[width=\linewidth]{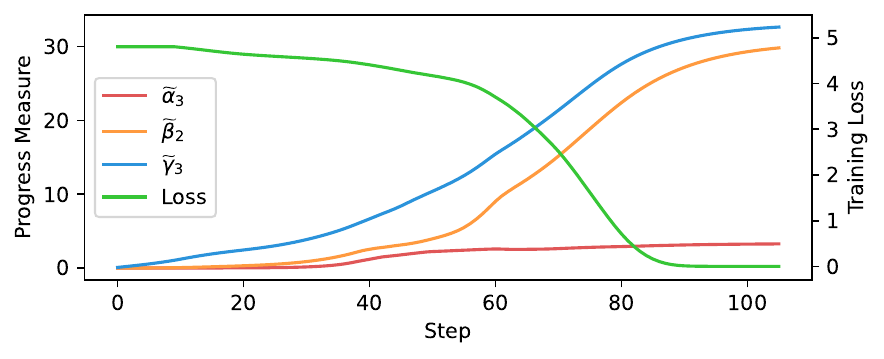}
  \caption{The emergence of induction heads in a standard attention-only transformer, illustrated using progress measures, resembles the disentangled transformer.}
  \label{fig:progress_measures}
\end{figure}

\section{Discussion}

\subsection{How do $\alpha$, $\beta$, and $\gamma$ emerge during training?}

\paragraph{The emergence of $\gamma$.} Even if $\alpha$ and $\beta$ are completely untrained, the attention layers still return something: the average of all items and labels in the context. This average achieves a better loss than predicting zero because it also contains the correct label, and this is exactly what the model learns to predict initially. However, this solution becomes worse when $N$ is increased. In fact, the gradient towards this solution is inversely proportional to $N$, hence why $\gamma$ emerges in $\Theta(N)$.

\paragraph{The emergence of $\beta$.} After the final layer is in place, there is now a gradient for the second layer to attend correctly. Because each label follows immediately after its item, the first layer will always retrieve the item to some extent, even when completely untrained. Taking the causal masking into account, each item will be retrieved the most by its label. This enables the second layer to learn to retrieve based on the query item. However, since the first layer returns a very weak signal (inversely proportional to $N$), the gradient of $\beta$ will be inversely proportional to $N^2$.

\paragraph{The emergence of $\alpha$.} Finally, after $\beta$ and $\gamma$ have emerged, there is a very strong gradient for the first layer to attend correctly. This quickly drives the emergence of $\alpha$.

\subsection{The Importance of Context Length}

We have established that a longer context length slows down the emergence of induction heads. This fact has interesting implications that are worth exploring in future work.

\citet{chan2022data} have empirically established that the emergence of in-context learning is modulated by data distributional properties specific to natural language, such as burstiness (items appear in clusters rather than being uniformly distributed
over time). Our work paves the way for a theoretical understanding of this connection. For example, bustiness could be understood as a modulator of the \textit{effective} context length by reducing the distance between items from the same class. We hypothesize that similar gains could be achieved by other means of reducing the \textit{effective} context length, such as special positional embeddings \citep{roformer}.

\section{Related Work}
\paragraph{In-Context Learning}
\citet{brown2020language} first observed that LLMs are capable of in-context learning. \citet{chan2022data} empirically showed that the ICL--IWL trade-off is modulated by data distributional properties specific to natural language. % TO ADD BACK! 
%Other theoretical works focus in-context linear regression \citet{lu2024asymptotic, zhang2025training} or argue that attention performs gradient descent \citep{von2023transformers, ahn2023transformers}.

\paragraph{Induction Heads}
 Later, \citet{olsson2022context} attributed this ability to a two-layer \citep{sanford2024one} mechanism (termed \textit{induction head}) that emerges abruptly during training. Crucial to our work, \citet{reddy2023mechanistic} proposed a 3-parameter \textit{phenomenological} model of an induction head by directly parameterizing the attention scores. The parameters of this model (denoted as $\beta_1$, $\alpha$, and $\xi$) correspond exactly to our three pseudo-parameters ($\alpha_3$, $\beta_2$, and $\gamma_3$). %Compared to their work, we provide a theoretical justification on how these parameters are learned with gradient descent.
 %Despite its simplicity, his model can account for the phenomena of abrupt learning and data distributional properties. 
 %Our work resolves this open question. 
 %Other theoretical works have studied the emergence of induction heads, with different architectures and distributional assumptions  \citep{nichani2024transformers,bietti2024birth, chen2024unveiling, sanford2024one, edelman2024evolution, wang2024transformers}.
 Several works
 \citep{nichani2024transformers, edelman2024evolution, chen2024unveiling, wang2024transformers} use a staged training process to show how transformers acquire the ability to learn in-context bigrams. \citet{bietti2024birth} study the transformer training dynamics from the perspective of \textit{associative memories}, showing how an induction head can emerge after three steps of gradient descent. %Concurrently, \citet{chen2024unveiling} and \citet{wang2024transformers} further studied staged layer-wise dynamics, reinforcing the staged learning hypothesis for induction head formation.%, and \citet{zhang2025training} analyzed training dynamics for linear attention transformers in regression tasks.
%\citet{chen2024unveiling} also study the training dynamics of induction head formation. They show that gradient flow converges to an induction head during a staged learning process. %\citet{wang2024transformers} also study the training dynamics of induction heads in a layer-wise staged paradigm.

%\citet{sanford2024one} shows that one-layer transformers cannot learn induction heads

\paragraph{Mechanistic Interpretability}
Several works seeks to attribute particular behaviors of neural networks to specific patterns in their weights and activations \citep{olah2020zoom, elhage2021mathematical, doshi2017towards, olah2017feature, bereska2024mechanistic, cammarata2020thread}. \citet{friedman2023learning} introduce the \textit{disentangled transformer} architecture, which is interpretable by design. Other works focus on multi-step reasoning \citep{wang2024buffer, musat2025mechanism, cabannes2024iteration}, context-free grammars \citep{allen2023physics}, and modular addition \citep{nanda2023progress, zhong2023clock, gromov2023grokking, he2024learning}.

\section{Conclusion}

In this paper, we have shown how induction heads emerge in an ICL task. Our work paves the way for a better theoretical understanding of transformer learning dynamics. We believe that a similar approach could illuminate other important phenomena in deep learning, such as the \textit{in-context vs. in-weights learning} trade-off, abrupt learning, or the emergence of other transformer circuits.

\section{Limitations}

While our work provides a formal proof for the emergence of induction heads, it relies on several simplifying assumptions.

First, our architectural scope is limited to a disentangled, attention-only transformer. While we empirically show similarities to standard transformers in \cref{sec:standard-transformers}, the presence of Layer Normalization, MLPs, and additive residual streams in production-scale models may introduce complex training dynamics not captured by our framework.

Second, our theoretical results assume orthonormal inputs and zero initialization. In practice, token embeddings are not perfectly orthogonal and evolve during training. This may alter the emergence time $t_{\text{ICL}}$ or the specific sequence of parameter growth.

Finally, our study focuses on a simplified ICL task (exact matching of item-label pairs). Natural language involves varying query positions, semantic nuances, contextual cues, and hierarchical structures. Understanding how the quadratic bound on emergence time translates to data distributions more similar to natural language remains an open question for future work.


\bibliography{custom}

%%%%%%%%%%%%%%%%%%%%%%%%%%%%%%%%%%%%%%%%%%%%%%%%%%%%%%%%%%%%%%%%%%%%%%%%%%%%%%%
%%%%%%%%%%%%%%%%%%%%%%%%%%%%%%%%%%%%%%%%%%%%%%%%%%%%%%%%%%%%%%%%%%%%%%%%%%%%%%%
% APPENDIX
%%%%%%%%%%%%%%%%%%%%%%%%%%%%%%%%%%%%%%%%%%%%%%%%%%%%%%%%%%%%%%%%%%%%%%%%%%%%%%%
%%%%%%%%%%%%%%%%%%%%%%%%%%%%%%%%%%%%%%%%%%%%%%%%%%%%%%%%%%%%%%%%%%%%%%%%%%%%%%%
\newpage
\appendix
\onecolumn

\section{Weights Structure Full Proof for \cref{thm:19-parameters}}
\label{app:weight_structure}

\subsection{Summary}

Our strategy is to show that if $W^{(1)}, W^{(2)}$, and $W^{(3)}$ have this structure, then their gradients also have the same structure. Since we start from zero initialization, by induction, this means that the structure is preserved throughout the entire training process.

To prove the structure of the gradient, we apply a carefully chosen rotation to the entire data distribution. Since the data distribution is isotropic, the rotation will not change the data distribution, so the expected gradient will also remain unchanged.

However, we are also able to show that our rotation induces a specific similarity transformation of the gradient:
\begin{equation*}
    \mathbb{E}\Biggl[ \frac{\partial \mathcal{L}}{\partial \W{k}[ij]} \Bigg]
    \;=\;
    F \; \mathbb{E}\Biggl[ \frac{\partial \mathcal{L}}{\partial \W{k}[ij]} \Bigg] \; F^\top
\end{equation*}
where $F$ is an orthogonal or block-orthogonal matrix and $\W{k}[ij]$ is a block of a weight matrix. From this we are able to show that the expected gradient must have the desired structure.  

\subsection{Prerequisites}

\subsubsection{Orthogonal Transformations}

\begin{definition} \label{def:orthogonal}
    \defn{Orthogonal Matrix.}
    We say that a matrix $E \in \R^{k \times k}$ is orthogonal if it satisfies $EE^\top = E^\top E = I$.
\end{definition}
\begin{proposition} \label{prop:orthogonal}
    Let $A \in \R^{k \times k}$ be some matrix. If $EAE^\top = A$ holds \textbf{for all} orthogonal matrices $E \in \R^{k \times k}$, then it follows that $A=\alpha \, I$ for some $\alpha \in \R$.
\end{proposition}
\begin{proof}
\emph{Step 1. All off‐diagonal entries of \(A\) vanish.}

Fix an index \(j\in\{1,\dots,k\}\) and let
\[
E \;=\;\mathrm{diag}(1,\dots,1,-1,1,\dots,1)
\]
be the diagonal orthogonal matrix with entry \(-1\) in the \(j\)th position and \(+1\) elsewhere.  Then
\[
(EAE^\top)_{i\ell}
\;=\;
E_{ii}\,A_{i\ell}\,E_{\ell\ell}
\;=\;
\begin{cases}
A_{i\ell},&i,\ell\neq j,\\
-\,A_{i\ell},&\text{exactly one of }i,\ell=j,\\
A_{jj},&i=\ell=j.
\end{cases}
\]
Since \(EAE^\top=A\), it follows that \(-A_{ij}=A_{ij}\) for every \(i\neq j\), whence \(A_{ij}=0\).  Varying \(j\) shows all off‐diagonal entries vanish, so
\[
A=\operatorname{diag}(a_{11},a_{22},\dots,a_{kk}).
\]

\medskip

\emph{Step 2. All diagonal entries of \(A\) coincide.}

Let \(E\) be any permutation matrix which swaps two coordinates \(i\) and \(j\).  Then \(E\) is orthogonal and
\[
EAE^\top
=\operatorname{diag}(\dots,a_{jj},\dots,a_{ii},\dots),
\]
interchanging the \(i\)th and \(j\)th diagonal entries of \(A\).  By invariance \(EAE^\top=A\), so \(a_{ii}=a_{jj}\).  Since \(i,j\) were arbitrary, there exists \(\alpha\in\mathbb{R}\) such that
\[
a_{11}=a_{22}=\cdots=a_{kk}=\alpha,
\]
and hence \(A=\alpha I\).

\end{proof}

\subsubsection{Block-Orthogonal Transformations}

\begin{definition} \label{def:block_orthogonal}
    \defn{Block-Orthogonal Matrix.}
    We say that a matrix $F \in \R^{2k \times 2k}$ is block-orthogonal if it has either of the following two forms:
    \begin{equation*}
        F = 
            \begin{bNiceArray}{c|c}[margin]
              {E} & \mathbf{0} \\
              \hline
              \mathbf{0} & {E}
            \end{bNiceArray}
        \qquad
        \text{or}
        \qquad
        F = 
            \begin{bNiceArray}{c|c}[margin]
              \mathbf{0} & {E} \\
              \hline
              {E} & \mathbf{0}
            \end{bNiceArray}
    \end{equation*}
    where $E \in \R^{k \times k}$ is an orthogonal matrix.
\end{definition}

\begin{proposition} \label{prop:block_orthogonal}
    Let $A \in \R^{2k \times 2k}$ be some matrix. If $FAF^\top = A$ holds \textbf{for all} block-orthogonal matrices $F \in \R^{2k \times 2k}$, then it follows that
    \begin{equation*}
        A = 
            \begin{bNiceArray}{c|c}[margin]
              {\alpha I} & {\beta I} \\
              \hline
              {\beta I} & {\alpha I}
            \end{bNiceArray}
    \end{equation*}
    for some $\alpha, \,\beta \,\in\, \R$.
\end{proposition}

\begin{remark}
    Note that this condition is weaker than the condition stated in \cref{prop:orthogonal}, since not all orthogonal matrices are also block-orthogonal. Hence, the condition in \cref{prop:block_orthogonal} guarantees a structure that is less specific than \cref{prop:orthogonal}.
\end{remark}

\begin{proof}
We write
\[
A \;=\;
\begin{bNiceArray}{c|c}[margin]
A_{11} & A_{12} \\ \hline
A_{21} & A_{22}
\end{bNiceArray},
\]
where each block \(A_{ij}\in\R^{k\times k}\).  

\medskip
\noindent\emph{Step 1. All blocks are scalar matrices.}

For any othogonal matrix \(E\), we can set
\[
F \;=\;
\begin{bNiceArray}{c|c}[margin]
E & \mathbf{0} \\ \hline
\mathbf{0} & E
\end{bNiceArray}.
\]
Then 
\[
F\,A\,F^\top
=
\begin{bNiceArray}{c|c}[margin]
E A_{11}E^\top & E A_{12}E^\top \\ \hline
E A_{21}E^\top & E A_{22}E^\top
\end{bNiceArray}
\;=\;
\begin{bNiceArray}{c|c}[margin]
A_{11} & A_{12} \\ \hline
A_{21} & A_{22}
\end{bNiceArray}
\;=\;A,
\]
so \(E\,A_{ij}\,E^\top = A_{ij}\) for all \(i,j\).  By the previous proposition each block is a scalar multiple of the identity, $A_{ij} = \alpha_{ij}\,I_k$, for some $\alpha_{ij}\in\R$. Therefore,
\[
A =
\begin{bNiceArray}{c|c}[margin]
\alpha_{11}\,I & \alpha_{12}\,I \\ \hline
\alpha_{21}\,I & \alpha_{22}\,I
\end{bNiceArray}.
\]

\medskip
\noindent\emph{Step 2. Diagonally opposed blocks coincide.}
By setting
\[
F \;=\;
\begin{bNiceArray}{c|c}[margin]
\mathbf{0} & I \\ \hline
I & \mathbf{0}
\end{bNiceArray}
\]
we obtain
\[
F\,A\,F^\top
\;=\;
\begin{bNiceArray}{c|c}[margin]
A_{22} & A_{21} \\ \hline
A_{12} & A_{11}
\end{bNiceArray}
%\;=\;
%A
%\;=\;
%\begin{bNiceArray}{c|c}
%A_{11} & A_{12} \\ \hline
%A_{21} & A_{22}
%\end{bNiceArray},
\]
which yields $\alpha_{11} = \alpha_{22}$, $\alpha_{12} = \alpha_{21}$. By writing
\(\alpha=\alpha_{11}\) and \(\beta=\alpha_{12}\), we obtain
\[
A =
\begin{bNiceArray}{c|c}[margin]
\alpha I & \beta I \\ \hline
\beta I & \alpha I
\end{bNiceArray}
\]
\end{proof}

\subsubsection{Combined Transformations}

\begin{proposition}  \label{prop:combined_orthogonal}
    Let $A \in \R^{2k \times 2k}$ be some matrix. If $EAF = A$ holds for all orthogonal matrices $E$ and block-orthogonal matrices $F$, then $A = \mathbf{0}$.
\end{proposition}
\begin{proof} By setting $E=I$ and $F=-I$, we get $A=-A$. Therefore, $A=\mathbf{0}$.
\end{proof}

\subsubsection{Block-Swap Transformation}

\begin{definition} \label{def:block_swap}
    \defn{Block-Swap Matrix.}
    We say that a matrix $M \in \R^{2k \times 2k}$ is block-swap if it has the following form:
    \begin{equation*}
        M = 
            \begin{bNiceArray}{c|c}[margin]
              \mathbf{0} & {I} \\
              \hline
              {I} & \mathbf{0}
            \end{bNiceArray}
    \end{equation*}
    where $I \in \R^{k \times k}$ is the identity matrix.
\end{definition}

\begin{proposition} \label{prop:block_swap}
    If $M \in \R^{2k \times 2k}$ is a block-swap matrix and $F \in \R^{2k \times 2k}$ is a block-orthogonal matrix, then $FMF\tp = M$.
\end{proposition}

\begin{proof}
\emph{Case 1. The orthogonal blocks of F are on the main diagonal.}

Assume that
\begin{equation*}
    F = \mat{E}{\zero}{\zero}{E}
\end{equation*}

Then,
\begin{align*}
    FMF\tp &\;=\; \mat{E}{\zero}{\zero}{E} \mat{\zero}{I}{I}{\zero} \mat{E\tp}{\zero}{\zero}{E\tp} \\
    &\;=\; \mat{\zero}{E}{E}{\zero} \mat{E\tp}{\zero}{\zero}{E\tp} \\
    &\;=\; \mat{\zero}{I}{I}{\zero}
\end{align*}

\medskip

\emph{Case 2. The orthogonal blocks of F are on the secondary diagonal.}

Assume that
\begin{equation*}
    F = \mat{\zero}{E}{E}{\zero}
\end{equation*}

Then,
\begin{align*}
    FMF\tp &\;=\; \mat{\zero}{E}{E}{\zero} \mat{\zero}{I}{I}{\zero} \mat{\zero}{E\tp}{E\tp}{\zero} \\
    &\;=\; \mat{E}{\zero}{\zero}{E} \mat{\zero}{E\tp}{E\tp}{\zero} \\
    &\;=\; \mat{\zero}{I}{I}{\zero}
\end{align*}
\end{proof}

\subsection{Setup}

Recall the architecture and loss:
\begin{align*}
U &= \Big[ \; X \; \Big| \; \sigma(X W^{(1)} X\tp) \, X \; \Big] &
V &= \Big[ \; U \; \Big| \; \sigma(U W^{(2)} U\tp) \, U \; \Big]
\end{align*}
\begin{align*}
z = V_{2N+1} W^{(3)} \qquad
\mathcal{L} = \| y - z \|^2
\end{align*}

where $\sigma$ to denotes the softmax function with causal masking, $[\; \cdot \;|\; \cdot \;]$ denotes matrix concatenation, and
\begin{align*}
&W^{(1)} \in \mathbb{R}^{2D \times 2D}&\qquad
&W^{(2)} \in \mathbb{R}^{4D \times 4D}&\qquad
&W^{(3)} \in \mathbb{R}^{8D \times D}& \\
&U \in \mathbb{R}^{(2N + 1) \times 4D}&
&V \in \mathbb{R}^{(2N + 1) \times 8D}&
&z \in \R^{D}\qquad&\\
&X \in \mathbb{R}^{(2N + 1) \times 2D}&
&y \in \mathbb{R}^{D}&
\end{align*}

The data is generated as:

\begin{align*}
X_{2i-1} = \big[ \; a_i \; | \; p_i \; \big]\qquad 
X_{2i} = \big[ \; b_i \; | \; p_iM \; \big]\qquad
\forall i \in \{ 1, \ldots, N \}
\end{align*}
\begin{align*}
X_{2N+1} = \big[ \; a_q \; | \; 0 \; \big]\qquad
y = b_q
\end{align*}

where

\begin{equation*}
    a_i, \;b_i, \;p_i \;\in\; \R^D
    \qquad
    q\in\{ 1, 2, \ldots, N\}
    \qquad
    M =
    \begin{bNiceArray}{cc}
      {\mathbf{0}} & I \\
      I & {\mathbf{0}}
    \end{bNiceArray}
\end{equation*}

All vectors are treated as \textbf{row vectors}.

\subsection{Additional Notation}

We introduce
\begin{align*}
&S = X \W{1} X\tp &\qquad
&T = \sigma(S)&\\
&P=U \W{2} U\tp &\qquad
&Q=\sigma(P)
\end{align*}
where $S,\, T,\, P,\, Q \,\in\, \R^{(2N+1)\times(2N+1)}$. This gives
\begin{align*}
U &= \Big[ \; X \; \Big| \; T X \; \Big] &
V &= \Big[ \; U \; \Big| \; Q U \; \Big]
\end{align*}

We also introduce notation for all blocks of size $D$:

\begin{equation*}
\begin{array}{c}
    X =
    \Big[\;
      {X_1} \quad {X_2}
    \;\Big]
    \qquad
    U =
    \Big[\;
      {U_1} \quad {U_2} \quad {U_3} \quad {U_4}
    \;\Big]
    
    \\[12pt]

    V =
    \Big[\;
      {V_1} \quad {V_2} \quad {V_3} \quad {V_4} \quad
      {V_5} \quad {V_6} \quad {V_7} \quad {V_8}
    \;\Big]
    
    \\[16pt]

    \W{1} =
    \begin{bNiceArray}{cc}
      \W{1}[11] & \W{1}[12] \\[8pt]
      %\hline
      \W{1}[21] & \W{1}[22]
    \end{bNiceArray}
    
    \qquad

    \W{2} =
    \begin{bNiceArray}{cccc}
      \W{2}[11] & \W{2}[12] & \W{2}[13] & \W{2}[14] \\[8pt]
      \W{2}[11] & \W{2}[12] & \W{2}[13] & \W{2}[14] \\[8pt]
      \W{2}[11] & \W{2}[12] & \W{2}[13] & \W{2}[14] \\[8pt]
      \W{2}[11] & \W{2}[12] & \W{2}[13] & \W{2}[14]
    \end{bNiceArray}
    
    \\[48pt]

    \W{3} =
    \begin{bNiceArray}{cccccccc}
      \W{3}[1] & \W{3}[2] & \W{3}[3] & \W{3}[4] & \W{3}[5] & \W{3}[6] & \W{3}[7] & \W{3}[8]
    \end{bNiceArray}

\end{array}
\end{equation*}

\subsection{Data Rotations}

We apply an orthogonal transformation $E$  to the items and labels, and a block-orthogonal transformation $F$ to the positional embeddings:

\begin{equation*}
    a'_i = a_iE
    \qquad
    b'_i = b_iE
    \qquad
    p'_i = p_iF
    \qquad
    \forall i \in \{\, 1 ,\, \ldots ,\, N \, \} 
\end{equation*}

where $E$ and $F$ satisfy \cref{def:orthogonal,def:block_orthogonal}, respectively. We refer to the new variables as $X'$, $y'$, $U'$, $V'$, $z'$, and $\Loss'$.

Since the data is isotropic, we have that $\E{\Loss} = \E{\Loss'}$. By the linearity of expectation and differentiation, we obtain
\begin{align*}
    \E{\df{\Loss}{\W{k}}} &\;=\; \E{\df{\Loss'}{\W{k}}}
\end{align*}

This also holds for all sub-blocks of $\W{1}$, $\W{2}$, and $\W{3}$,
\begin{align*}
    \E{\df{\Loss}{\W{k}[ij]}}  &\;=\; \E{\df{\Loss'}{\W{k}[ij]}}
\end{align*}

However, as we show below, our rotation induces specific transformations of the gradient blocks. Using \cref{prop:orthogonal,prop:block_orthogonal,prop:combined_orthogonal}, we are able to show that each gradient block has the desired structure.

Specifically, for each gradient block, we will show that one of the following four conditions holds for all $E$ and $F$, implying the desired structure:
\begin{align*}
    \E{\df{\Loss'}{\W{k}[ij]}}  \;=\; E \;\E{\df{\Loss'}{\W{k}[ij]}} E\tp
    &\quad\Longrightarrow\quad
    \E{\df{\Loss}{\W{k}[ij]}}  \;=\; \alpha \, I
    \\[8pt]
    \E{\df{\Loss'}{\W{k}[ij]}}  \;=\; F \;\E{\df{\Loss'}{\W{k}[ij]}} F\tp
    &\quad\Longrightarrow\quad
    \E{\df{\Loss}{\W{k}[ij]}}  \;=\; 
    \begin{bNiceArray}{cc}[margin]
      \alpha \, I & \beta \, I \\[8pt]
      %\hline
      \beta \, I & \alpha \, I
    \end{bNiceArray}
    \\[8pt]
    \E{\df{\Loss'}{\W{k}[ij]}}  \;=\; E \;\E{\df{\Loss'}{\W{k}[ij]}} F\tp
    &\quad\Longrightarrow\quad
    \E{\df{\Loss}{\W{k}[ij]}}  \;=\; \mathbf{0}
    \\[8pt]
    \E{\df{\Loss'}{\W{k}[ij]}}  \;=\; F \;\E{\df{\Loss'}{\W{k}[ij]}} E\tp
    &\quad\Longrightarrow\quad
    \E{\df{\Loss}{\W{k}[ij]}}  \;=\; \mathbf{0}
\end{align*}

\subsection{Forward Pass}

We will now observe how our rotation changes the intermediate and final results of our model.

First, note the rotated inputs and outputs:
\begin{equation*}
    X'_1 = X_1E \qquad
    X'_2 = X_2F \qquad
    y' = yE
\end{equation*}

Recall that we are assuming that $\W{1}$, $\W{2}$, and $\W{3}$ already have the desired structure, with the goal to prove that the gradient has the same structure:
\begin{align*}
    W^{(1)} &\;=\;
    \begin{bNiceArray}{c|c}[margin]
      {\alpha_1 I} & \mathbf{0} \\
      \hline
      \mathbf{0} & {\alpha_2 I + \alpha_3 M}
    \end{bNiceArray}
    \\[4pt]
    W^{(2)} &\;=\;
    \begin{bNiceArray}{c|c|c|c}[margin]
      {\beta_1 I} & \mathbf{0} & {\beta_2 I} & \mathbf{0} \\
      \hline
      \mathbf{0} & {\beta_3 I + \beta_4 M} & \mathbf{0} & {\beta_5 I + \beta_6 M} \\
      \hline
      {\beta_7 I} & \mathbf{0} & {\beta_8 I} & \mathbf{0} \\
      \hline
      \mathbf{0} & {\beta_9 I + \beta_{10} M} & \mathbf{0} & {\beta_{11} I + \beta_{12} M}
    \end{bNiceArray}
    \\[4pt]
    W^{(3)} &\;=\;
    \begin{bNiceArray}{c|c|c|c|c|c|c|c}[margin]
      {\gamma_1 I} & \mathbf{0} &
      {\gamma_2 I} & \mathbf{0} &
      {\gamma_3 I} & \mathbf{0} &
      {\gamma_4 I} & \mathbf{0}
    \end{bNiceArray}
\end{align*}

\subsubsection{First Layer}

The first attention layer gives:
\begin{align*}
    S &\;=\; X \W{1}X\tp \\
    &\;=\; X_{1} \W{1}[11] X_{1}\tp + X_{1} \W{1}[12] X_{2}\tp + X_{2} \W{1}[21] X_{1}\tp + X_{2} \W{1}[22] X_{2}\tp \\
    &\;=\; \alpha_1 X_{1} X_{1}\tp + \alpha_2 X_{2} X_{2}\tp + \alpha_3 X_{2} M X_{2}\tp
    \\[4pt]
    S' &\;=\; X' \W{1}{X'}\tp \\
    &\;=\; X'_{1} \W{1}[11] {X'_{1}}\tp + X'_{1} \W{1}[12] {X'_{2}}\tp + X'_{2} \W{1}[21] {X'_{1}}\tp + X'_{2} \W{1}[22] {X'_{2}}\tp \\
    &\;=\; \alpha_1 X'_{1} {X'_{1}}\tp + \alpha_2 X'_{2} {X'_{2}}\tp + \alpha_3 X'_{2} M {X'_{2}}\tp\\
    &\;=\; \alpha_1 X_{1} E E\tp X_{1}\tp + \alpha_2 X_{2} F F\tp X_{2}\tp + \alpha_3 X_{2} F M F\tp X_{2}\tp\\
    &\;=\; \alpha_1 X_{1} X_{1}\tp + \alpha_2 X_{2} X_{2}\tp + \alpha_3 X_{2} M X_{2}\tp
\end{align*}

Therefore, $S' = S$ and $T' = T = \sigma(S)$. This gives us:
\begin{equation*}
    U'_1 = U_1E \qquad
    U'_2 = U_2F \qquad
    U'_3 = U_3E \qquad
    U'_4 = U_4F
\end{equation*}

\subsubsection{Second Layer}
The second attention layer gives:
\begin{align*}
    P &\;=\; U \W{2}U\tp \\
    &\;=\; \sum U_i \W{2}[ij] {U_j}\tp  \\
    &\;=\; \beta_1U_1{U_1}\tp + \beta_2U_1U_3\tp + \beta_7U_3U_1\tp + \beta_8U_3U_3\tp \\&\qquad+\;
            \beta_3U_2U_2\tp + \beta_5U_2U_4\tp + \beta_9U_4U_2\tp + \beta_{11}U_4U_4\tp
    \\&\qquad+\;
            \beta_4U_2MU_2\tp + \beta_6U_2MU_4\tp + \beta_{10}U_4MU_2\tp + \beta_{12}U_4MU_4\tp
    \\[4pt]
    P' &\;=\; U' \W{2}{U'}\tp \\
    &\;=\; \sum U'_i \W{2}[ij] {U'_j}\tp  \\
    &\;=\; \beta_1U'_1{U'_1}\tp + \beta_2U'_1{U'_3}\tp + \beta_7U'_3{U'_1}\tp + \beta_8U'_3{U'_3}\tp \\&\qquad+\;
            \beta_3U'_2{U'_2}\tp + \beta_5U'_2{U'_4}\tp + \beta_9U'_4{U'_2}\tp + \beta_{11}U'_4{U'_4}\tp
    \\&\qquad+\;
            \beta_4U'_2M{U'_2}\tp + \beta_6U'_2M{U'_4}\tp + \beta_{10}U'_4M{U'_2}\tp + \beta_{12}U'_4M{U'_4}\tp  \\
    &\;=\; \beta_1U_1 EE\tp {U_1}\tp + \beta_2U_1 EE\tp U_3\tp + \beta_7U_3 EE\tp U_1\tp + \beta_8U_3 EE\tp U_3\tp
    \\&\qquad+\;
            \beta_3U_2 FF\tp U_2\tp + \beta_5U_2 FF\tp U_4\tp + \beta_9U_4 FF\tp U_2\tp + \beta_{11}U_4 FF\tp U_4\tp
    \\&\qquad+\;
            \beta_4U_2 FMF\tp U_2\tp + \beta_6U_2 FMF\tp U_4\tp + \beta_{10}U_4 FMF\tp U_2\tp + \beta_{12} FMF\tp U_4\tp \\
    &\;=\; \beta_1U_1{U_1}\tp + \beta_2U_1U_3\tp + \beta_7U_3U_1\tp + \beta_8U_3U_3\tp \\&\qquad+\;
            \beta_3U_2U_2\tp + \beta_5U_2U_4\tp + \beta_9U_4U_2\tp + \beta_{11}U_4U_4\tp
    \\&\qquad+\;
            \beta_4U_2MU_2\tp + \beta_6U_2MU_4\tp + \beta_{10}U_4MU_2\tp + \beta_{12}U_4MU_4\tp
\end{align*}

Therefore, $P' = P$ and $Q' = Q = \sigma(P)$. This gives us:
\begin{align*}
    V'_1 = V_1E \qquad
    V'_2 = V_2F \qquad
    V'_3 = V_3E \qquad
    V'_4 = V_4F \\[4pt]
    V'_5 = V_5E \qquad
    V'_6 = V_6F \qquad
    V'_7 = V_7E \qquad
    V'_8 = V_8F
\end{align*}

\subsubsection{Output Layer}

Finally, the output layer gives:
\begin{align*}
    z &\;=\; V_{2N+1}\W{3} \\
    &\;=\; \sum(V_i)_{2N+1}\W{3}[i] \\
    &\;=\; \gamma_1(V_1)_{2N+1} + \gamma_2(V_3)_{2N+1} + \gamma_3(V_5)_{2N+1} + \gamma_4(V_7)_{2N+1}
    \\[4pt]
    z' &\;=\; V'_{2N+1}\W{3} \\
    &\;=\; \sum(V'_i)_{2N+1}\W{3}[i] \\
    &\;=\; \gamma_1(V'_1)_{2N+1} + \gamma_2(V'_3)_{2N+1} + \gamma_3(V'_5)_{2N+1} + \gamma_4(V'_7)_{2N+1} \\
    &\;=\; \gamma_1(V_1)_{2N+1}E + \gamma_2(V_3)_{2N+1}E + \gamma_3(V_5)_{2N+1}E + \gamma_4(V_7)_{2N+1}E \\
    &\;=\; zE
\end{align*}

\subsection{Backward Pass}

We now show how the rotation transforms the gradient of each weight block.

\subsubsection{Output Layer}

\begin{align*}
\df{\Loss}{z} &\;=\; 2(z - y) \\[4pt]
\df{\Loss'}{z'} &\;=\;2(z'-y') \;=\; 2(zE - yE) \;=\; 2(z-y)E \;=\; \df{\Loss}{z}E\\[8pt]
\df{\Loss}{\W{3}[i]} &\;=\; \left(\left(V_i \right)_{2N + 1}\right)\tp \left( \df{\Loss}{z} \right) \\[4pt]
\df{\Loss'}{\W{3}[i]} &\;=\; \left(\left(V'_i \right)_{2N + 1}\right)\tp \left( \df{\Loss'}{z'} \right)
\end{align*}

\vspace{20pt}
\textbf{Scalar Blocks.} For all $i \in \{1, 3,5,7\}$, we get
\begin{align*}
    \df{\Loss'}{\W{3}[i]} &\;=\; \left(\left(V'_i \right)_{2N + 1}\right)\tp \left( \df{\Loss'}{z'} \right) \\
    &\;=\; E\tp\left(\left(V_i \right)_{2N + 1}\right)\tp \left( \df{\Loss}{z} \right)E \\
    &\;=\; E\tp \df{\Loss}{\W{3}[i]}E \\[4pt]
\end{align*}
Taking the expectation over the entire data distribution, we obtain that the following holds for any orthogonal transformation $E$:
\begin{align*}
    \E{\df{\Loss}{\W{3}[i]}} \;=\;
    \E{\df{\Loss'}{\W{3}[i]}} \;=\;
    \E{E\tp\df{\Loss'}{\W{3}[i]}E} \;=\;
    E\tp \, \E{\df{\Loss}{\W{3}[i]}} \, E
\end{align*}

Applying \cref{prop:orthogonal}, we get that
\begin{align*}
    \E{\df{\Loss}{\W{3}[i]}} \;=\; \alpha \, I
\end{align*}

\vspace{20pt}

\textbf{Zero Blocks.} For all $i \in \{2,4,6,8\}$, we get
\begin{align*}
    \df{\Loss'}{\W{3}[i]} &\;=\; \left(\left(V'_i \right)_{2N + 1}\right)\tp \left( \df{\Loss'}{z'} \right) \\
    &\;=\; F\tp\left(\left(V_i \right)_{2N + 1}\right)\tp \left( \df{\Loss}{z} \right)E \\
    &\;=\; F\tp \df{\Loss}{\W{3}[i]}E
\end{align*}

Taking the expectation over the entire data distribution, we obtain that the following holds for any $E$ and $F$:
\begin{align*}
    \E{\df{\Loss}{\W{3}[i]}} \;=\;
    F\tp \, \E{\df{\Loss}{\W{3}[i]}} \, E
\end{align*}

Applying \cref{prop:combined_orthogonal}, we get that
\begin{align*}
    \E{\df{\Loss}{\W{3}[i]}} \;=\; \zero
\end{align*}

\vspace{20pt}

\textbf{Gradient Propagation} Applying the chain rule, we get
\begin{align*}
    \df{\Loss}{V_{2N+1}} \;=\; \df{\Loss}{z} \, \df{z}{V_{2N+1}}  \;=\; 2(z-y) \, \W{3}\tp 
\end{align*}

For all $i \in \{1,3,5,7\}$, we get
\begin{align*}
    \df{\Loss'}{(V'_i)_{2N+1}} \;=\; 2(z' - y') \, \W{3}[i] \tp \;=\; 2(z - y) \, E \, \W{3}[i] \tp
     \;=\; \df{\Loss}{(V_i)_{2N+1}} E
\end{align*}

For all $i \in \{2,4,6,8\}$, we get
\begin{align*}
    \df{\Loss}{(V_i)_{2N+1}} \;=\; 2(z - y) \, \W{3}[i]
     \;=\; \zero
\end{align*}

For all $i \in \{1,\ldots,8\}$ and $j \le2N$, we get
\begin{align*}
    \df{\Loss}{(V_i)_{j}} \;=\;
    \df{\Loss}{(V'_i)_{j}} \;=\;
     \zero
\end{align*}

Putting everything together, the following holds for all $j \le 2N + 1$
\begin{align} \label{eq:grad_v_odd}
    &\df{\Loss}{(V_i)_{j}} \;=\;
    \df{\Loss'}{(V'_i)_{j}} E
    &
    \text{if }\;i \in \{1,3,5,7\}
    \\[8pt] \label{eq:grad_v_even}
    &\df{\Loss}{(V_i)_{j}} \;=\;
    \df{\Loss'}{(V'_i)_{j}} \;=\; \zero
    &
    \text{if }\;i \in \{2,4,6,8\}
\end{align}

\subsubsection{Second Layer}

Since $V = [ \; U \;| \; QU \;  ]$, we have that
\begin{equation*}
\df{(V_i)_j}{Q_{jk}} \;=\;
\begin{dcases*}
    U_{i-4} & $i>4$\\
    \zero & $i \le 4$
\end{dcases*}
\end{equation*}

Therefore,
\begin{align*}
&\df{\Loss'}{(V'_i)_j} \df{(V'_i)_j}{Q'_{jk}} \;=\;
\df{\Loss}{(V_i)_j} \; E\tp E \; \df{(V_i)_j}{Q_{jk}} \;=\;
\df{\Loss}{(V_i)_j} \df{(V_i)_j}{Q_{jk}} &i \in \{5,7\}
\\[8pt]
&\df{\Loss'}{(V'_i)_j} \df{(V'_i)_j}{Q'_{jk}} \;=\;
\df{\Loss}{(V_i)_j} \df{(V_i)_j}{Q_{jk}}  \;=\;
\zero \quad &i \not\in \{5,7\}
\end{align*}

Additionally, since $P' = P$ and $Q'=Q$, we have that
\begin{equation*}
    \df{Q'_{jk}}{P'_{jl}} = \df{Q_{jk}}{P_{jl}}
\end{equation*}

This gives us
\begin{align*}
    \df{\Loss'}{P'_{kl}} &\;=\; 
    \sum_{ij} \df{\Loss'}{(V'_i)_j} \df{(V'_i)_j}{Q'_{kj}} \df{Q'_{kj}}{P'_{kl}} 
    \\&\;=\;
    \sum_{ij} \df{\Loss}{(V_i)_k} \df{(V_i)_k}{Q_{kj}} \df{Q_{kj}}{P_{kl}} \\
    &\;=\;
    \df{\Loss}{P_{kl}}
\end{align*}

Additionally,
\begin{align*}
\df{\Loss}{\W{2}[ij]}
&\;=\; \sum_{kl} \df{\Loss}{P_{kl}} \df{P_{kl}}{\W{2}[ij]}
\;=\; \sum_{kl} \df{\Loss}{P_{kl}} (U_i)_k\tp \, (U_j)_l
\end{align*}
which gives us the transformed gradient:
\begin{align*}
\df{\Loss'}{\W{2}[ij]}
&\;=\; \sum_{kl} \df{\Loss'}{P'_{kl}} \df{P'_{kl}}{\W{2}[ij]}
\;=\; \sum_{kl} \df{\Loss}{P_{kl}} (U'_i)_k\tp \, (U'_j)_l\\[8pt]
&\;=\; \begin{dcases*}
    E\tp \df{\Loss}{\W{2}[ij]} E &\text{if $i$ odd, $j$ odd}\\
    E\tp \df{\Loss}{\W{2}[ij]} F &\text{if $i$ odd, $j$ even}\\
    F\tp \df{\Loss}{\W{2}[ij]} E &\text{if $i$ even, $j$ odd}\\
    F\tp \df{\Loss}{\W{2}[ij]} F &\text{if $i$ even, $j$ even}\\
\end{dcases*}
\end{align*}

The desired structure follows from computing the expected gradient over the entire distribution and applying \cref{prop:orthogonal,prop:block_orthogonal,prop:combined_orthogonal}.

\vspace{20pt}

\textbf{Gradient Propagation}

Applying the chain rule, we get
\begin{equation} \label{eq:grad_u}
    \df{\Loss'}{(U'_i)_j} \;=\; \df{\Loss'}{(V'_i)_j} + {Q'}\tp \df{\Loss'}{(V'_{i+4})_j} + \sum _{k}\df{\Loss'}{P'_{jk}} \df{P'_{jk}}{(U'_i)_j}
\end{equation}

We also have that
\begin{equation} \label{eq:dp_du}
    \df{P_{jk}}{(U_i)_j} \;=\; (U_i)_k \qquad\qquad
    \df{P'_{jk}}{(U'_i)_j} \;=\; (U'_i)_k \;=\; \begin{dcases*}
        (U_i)_k \; E &\text{if $i$ odd}\\[8pt]
        (U_i)_k \; F &\text{if $i$ even}
    \end{dcases*}
\end{equation}

Cobmining \cref{eq:grad_v_odd,eq:grad_v_even,eq:grad_u,eq:dp_du}, we get
\begin{equation*}
    \df{\Loss'}{(U'_i)_j} \;=\; 
\begin{dcases*}
    \df{\Loss'}{(U'_i)_j} E &\text{if $i$ odd}\\[8pt]
    \df{\Loss'}{(U'_i)_j} F &\text{if $i$ even}
\end{dcases*}
\end{equation*}

\subsubsection{First Layer}

Through similar derivations as before, we obtain
\begin{align*}
    \df{\Loss'}{S'_{kl}} &\;=\; 
    \sum_{ij} \df{\Loss'}{(U'_i)_j} \df{(U'_i)_j}{T'_{kj}} \df{T'_{kj}}{S'_{kl}} 
    \\&\;=\;
    \sum_{ij} \df{\Loss}{(U_i)_k} \df{(U_i)_k}{T_{kj}} \df{T_{kj}}{S_{kl}} \\
    &\;=\;
    \df{\Loss}{S_{kl}}
\end{align*}
and
\begin{align*}
\df{\Loss}{\W{1}[ij]}
&\;=\; \sum_{kl} \df{\Loss}{S_{kl}} \df{S_{kl}}{\W{1}[ij]}
\;=\; \sum_{kl} \df{\Loss}{S_{kl}} (X_i)_k\tp \, (X_j)_l
\end{align*}

This gives us the transformed gradient:
\begin{align*}
\df{\Loss'}{\W{1}[11]} \;=\; E\tp \df{\Loss}{\W{1}[11]}E &\hspace{20pt}&
\df{\Loss'}{\W{1}[12]} \;=\; E\tp \df{\Loss}{\W{1}[12]}F\\[8pt]
\df{\Loss'}{\W{1}[21]} \;=\; F\tp \df{\Loss}{\W{1}[21]}E &&
\df{\Loss'}{\W{1}[22]} \;=\; F\tp \df{\Loss}{\W{1}[22]}F
\end{align*}
which implies the desired structure.

% ==============================================================

\newpage
\section{Tight Bound Proof for \cref{thm:tight-bound}}
\label{app:tight_bound}

\subsection{Summary}

To simplify notation within this section, we refer to the parameters $\alpha_3, \beta_2$, and $\gamma_3$, simply as $\alpha$, $\beta$, and $\gamma$, respectively. We show that $\gamma$ is the first parameter to reach the value $1/2$ after a time $T_1 = \Theta(N)$, then remains bounded. Later, $\beta$ reaches $1/2$ after an additional time $T_2 = \Theta(N^2)$. Finally, $\alpha$ reaches $1/2$ after an additional time $T_3 = O(N^2)$. This gives the total times $T_{\alpha} = T_1 = \Theta(N)$, $T_{\beta} = T_1 + T_2 = \Theta(N) + \Theta(N^2) = \Theta(N^2)$, and $T_{\gamma} = T_1 + T_2 + T_3 = \Theta(N) + \Theta(N^2) + O(N^2) = \Theta(N^2)$. Each step is proven by appropriately bounding the gradient updates. We give the full proof below.

\subsection{Setup}

Recall the architecture and loss:

\begin{align*}
U &= \Big[ \; X \; \Big| \; \sigma(X W^{(1)} X\tp) X \; \Big] &
V &= \Big[ \; U \; \Big| \; \sigma(U W^{(2)} U\tp) U \; \Big]
\end{align*}
\begin{align*}
z = V_{2N+1} \, W^{(3)} \qquad
\mathcal{L} = \| y - z \|^2
\end{align*}

where $[\; \cdot \;|\; \cdot \;]$ denotes matrix concatenation, and
\begin{align*}
&W^{(1)} \in \mathbb{R}^{2D \times 2D}&\qquad
&W^{(2)} \in \mathbb{R}^{4D \times 4D}&\qquad
&W^{(3)} \in \mathbb{R}^{8D \times D}& \\
&U \in \mathbb{R}^{(2N + 1) \times 4D}&
&V \in \mathbb{R}^{(2N + 1) \times 8D}&
&z \in \R^{D}\qquad&\\
&X \in \mathbb{R}^{(2N + 1) \times 2D}&
&y \in \mathbb{R}^{D}&
\end{align*}

We use $\sigma$ to denote the softmax function with causal masking. We apply a causal mask that prevents a position from attending to itself, which is not a standard practice, but it greatly simplifies the proofs.

The data is generated as:

\begin{align*}
X_{2i-1} = \big[ \; a_i \; | \; p_i \; \big]\qquad 
X_{2i} = \big[ \; b_i \; | \; M p_i \; \big]\qquad
\forall i \in \{ 1, \ldots, N \}
\end{align*}
\begin{align*}
X_{2N+1} = \big[ \; a_q \; | \; 0 \; \big]\qquad
y = b_q
\end{align*}

where

\begin{equation*}
    a_i, \;b_i, \;p_i \;\in\; \R^D
    \qquad
    q\in\{ 1, 2, \ldots, N\}
    \qquad
    M =
    \begin{bNiceArray}{c|c}
      {\mathbf{0}} & I \\
      \hline
      I & {\mathbf{0}}
    \end{bNiceArray}
\end{equation*}

\subsection{Loss Function}

We begin by deriving a closed-form expression of the loss in terms of the three parameters.

The orthonormal inputs give us the following attention scores in the first layer:
\begin{equation*}
    (XW^{(1)}X^\top)_{ij} = \begin{cases}
\alpha & i = 2k,\; j=i-1\\
0 &\text{otherwise}
\end{cases}
\end{equation*}

Applying the softmax attention with causal masking gives us:
\begin{equation*}
\label{eq:attention_layer1}
    \sigma(XW^{(1)}X^\top)_{ij} = \begin{dcases*}
\frac{e^\alpha}{i - 2 + e^\alpha} & $i = 2k,\; j=i-1$\\
\frac{1}{i - 2 + e^\alpha} & $i = 2k,\; j \ne i-1$\\
\frac{1}{i - 1} & $i = 2k+1$
\end{dcases*}
\end{equation*}

From \cref{asspt:query_last}, the target label is the last element in the sequence, following immediately after the queried item. This means that only the target label will contain the queried item after the first layer. Therefore, the target label will be the only position attended by the query:
\begin{equation*}
    (UW^{(2)}U^\top)_{2N+1,\,i} = \begin{cases}
        \beta\frac{e^\alpha}{2N-2+e^\alpha} & i = 2N\\
        0 &\text{otherwise}
\end{cases}
\end{equation*}

Applying the softmax attention gives:
\begin{equation*}
    \sigma(UW^{(2)}U^\top)_{2N+1,\,i} = \begin{cases}
        \frac{s}{s+2N-1} & i = 2N\\
        \frac{1}{s+2N-1} &\text{otherwise}
\end{cases}
\end{equation*}
where $s = e^{\beta\frac{e^\alpha}{2N-2+e^\alpha}}$.

Applying the output projection layer will give us:
\begin{equation*}
    z = \frac{\gamma}{s+2N-1}\Biggl( s\,b_N + a_N + \sum_{i=1}^{N-1}\Big( a_i + b_i \Big) \Biggr)
\end{equation*}

The final loss will be:
\begin{align*}
    \mathcal{L} &= \|z-b_i \|^2 = \|z\|^2 - 2\,z^\top b_i + \|b_i\|^2 \\
                &= \gamma^2 \frac{s^2 + 2N - 1}{(s + 2N - 1)^2} - 2\gamma \frac{s}{s + 2N - 1} + 1
\end{align*}
where $s = e^{\beta\frac{e^\alpha}{2N-2+e^\alpha}}$.

Note that as long as inputs are orthonormal and the target label is in the last position, the loss only depends on $\alpha, \beta, \gamma,$ and $N$. Any distribution over orthonormal inputs will give the same expected loss.

\subsection{Loss Gradient}

We now proceed to compute the partial derivatives of the loss function with respect to each of the three parameters.
\medskip
\subsubsection{Auxiliary definitions}  
\begin{align*}
G &= e^{\alpha} + 2N - 2, 
&
F &= 2N - 1, \\[6pt]
s &= \exp\!\Bigl(\frac{\beta\,e^{\alpha}}{G}\Bigr),
&
r &= s + F, \\[6pt]
\mathcal{L} &= \frac{\gamma^{2}\,(s^{2} + F)}{r^{2}}
            \;-\; 2\gamma\,\frac{s}{r}
            \;+\; 1.
\end{align*}

\medskip

\subsubsection{Partial derivative w.r.t.\ \(\gamma\)}  
\begin{align*}
\frac{\partial \mathcal{L}}{\partial \gamma}
&= \frac{\partial}{\partial \gamma}\Bigl(\tfrac{\gamma^{2}(s^{2}+F)}{r^{2}} - 2\gamma\,\tfrac{s}{r} + 1\Bigr) \\[6pt]
&= 2\gamma\,\frac{s^{2} + F}{r^{2}}
   \;-\; 2\,\frac{s}{r}
\end{align*}

\bigskip

\subsubsection{Partial derivative w.r.t.\ \(s\)}  
\begin{align*}
\frac{\partial \mathcal{L}}{\partial s}
&= \frac{\partial}{\partial s}\Bigl(\tfrac{\gamma^{2}(s^{2}+F)}{r^{2}}\Bigr)
  \;-\; 2\gamma\,\frac{\partial}{\partial s}\Bigl(\tfrac{s}{r}\Bigr) \\[6pt]
&= \gamma^{2}\,\frac{2s\,r^{2} - (s^{2}+F)\,2r\,\frac{\partial r}{\partial s}}{r^{4}}
  \;-\; 2\gamma\,\frac{r - s\,\frac{\partial r}{\partial s}}{r^{2}}
\end{align*}
But since $\partial r/\partial s = 1$,
\begin{align*}
\frac{\partial \mathcal{L}}{\partial s} &= \frac{2\gamma^{2}s\,r - 2\gamma^{2}(s^{2}+F)}{r^{3}}
  \;-\; 2\gamma\,\frac{r - s}{r^{2}} \\[6pt]
&= 2F\,\Bigl(\frac{\gamma^{2}(s-1)}{r^{3}} - \frac{\gamma}{r^{2}}\Bigr)
\end{align*}

\bigskip

\subsubsection{Derivatives of \(s\)}  
\[
s = \exp\!\Bigl(\tfrac{\beta\,e^{\alpha}}{G}\Bigr)
\quad\Longrightarrow\quad
\begin{cases}
\displaystyle \frac{\partial s}{\partial \alpha}
= s\,\frac{\partial}{\partial \alpha}\Bigl(\tfrac{\beta\,e^{\alpha}}{G}\Bigr)
= s\,\frac{\beta\,e^{\alpha}(G - e^{\alpha})}{G^{2}}
= s\,\frac{2(N-1)\,\beta\,e^{\alpha}}{G^{2}} \\[8pt]
\displaystyle \frac{\partial s}{\partial \beta}
= s\,\frac{\partial}{\partial \beta}\Bigl(\tfrac{\beta\,e^{\alpha}}{G}\Bigr)
= s\,\frac{e^{\alpha}}{G}
\end{cases}
\]

\bigskip

\subsubsection{Applying the chain‐rule results}  
\begin{align*}
\frac{\partial \mathcal{L}}{\partial \alpha}
&= \frac{\partial \mathcal{L}}{\partial s}\,\frac{\partial s}{\partial \alpha}
= 2F\,\Bigl(\tfrac{\gamma^{2}(s-1)}{r^{3}} - \tfrac{\gamma}{r^{2}}\Bigr)
  \;\times\; s\,\frac{2(N-1)\,\beta\,e^{\alpha}}{G^{2}} \\[6pt]
&= \frac{4\,\beta\,(N-1)\,F\,s\,e^{\alpha}}{G^{2}}
  \;\Bigl(\frac{\gamma^{2}(s-1)}{r^{3}} - \frac{\gamma}{r^{2}}\Bigr)
\\[12pt]
\frac{\partial \mathcal{L}}{\partial \beta}
&= \frac{\partial \mathcal{L}}{\partial s}\,\frac{\partial s}{\partial \beta}
= 2F\,\Bigl(\tfrac{\gamma^{2}(s-1)}{r^{3}} - \tfrac{\gamma}{r^{2}}\Bigr)
  \;\times\; s\,\frac{e^{\alpha}}{G} \\[6pt]
&= \frac{2\,F\,s\,e^{\alpha}}{G}
  \;\Bigl(\frac{\gamma^{2}(s-1)}{r^{3}} - \frac{\gamma}{r^{2}}\Bigr)
\end{align*}

\bigskip

\subsubsection{Final results}  
\begin{align*}
\frac{\partial \mathcal{L}}{\partial \alpha}
\;&=\; \frac{4\,\beta\,(N-1)\,F\,s\,e^{\alpha}}{G^{2}}
  \;\Bigl(\frac{\gamma^{2}(s-1)}{r^{3}} - \frac{\gamma}{r^{2}}\Bigr)
\\[12pt]
\frac{\partial \mathcal{L}}{\partial \beta}
\;&=\; \frac{2\,F\,s\,e^{\alpha}}{G}
  \;\Bigl(\frac{\gamma^{2}(s-1)}{r^{3}} - \frac{\gamma}{r^{2}}\Bigr)
\\[12pt]
\frac{\partial \mathcal{L}}{\partial \gamma}
\;&=\; 2\gamma\,\frac{s^{2} + F}{r^{2}}
   \;-\; 2\,\frac{s}{r}
\end{align*}

\subsubsection{Verification}

We verify the correctness of the previous results using automated symbolic differentiation with the \textit{SymPy} library. The code is provided with this paper. 

\subsection{Emergence of In-Context Learning}

Combining the previously obtained loss derivatives with the zero initialization, we obtain the full set of constraints that determine our training trajectory:

\begin{equation*}
    \alpha(0) = \beta(0) = \gamma(0) = 0
\end{equation*}
\begin{align*}
\frac{\partial \alpha}{\partial t}
\;&=\; \frac{2\,\beta\,(2N-2)\,(2N-1)\,s\,e^{\alpha}}{(e^\alpha + 2N -2)^{2}}
  \;\Biggl(\frac{\gamma}{(s + 2N-1)^{2}} - \frac{\gamma^{2}(s-1)}{(s + 2N-1)^{3}} \Biggr)
\\[12pt]
\frac{\partial \beta}{\partial t}
\;&=\; \frac{2\,(2N-1)\,s\,e^{\alpha}}{e^\alpha + 2N -2}
  \;\Biggl(\frac{\gamma}{(s + 2N-1)^{2}} - \frac{\gamma^{2}(s-1)}{(s + 2N-1)^{3}} \Biggr)
\\[12pt]
\frac{\partial \gamma}{\partial t}
\;&=\; 2\,\frac{s}{s + 2N-1} \;-\; 2\gamma\,\frac{s^{2} + 2N - 1}{(s + 2N-1)^{2}}
\end{align*}
where $s = \exp\!\Bigl(\beta\,\frac{e^{\alpha}}{e^\alpha + 2N - 2}\Bigr)$.  

We are interested in the first time $t_{\text{ICL}}$ when all three parameters are greater than $1/2$. As we show below, the parameters always reach this value in a specific order: first $\gamma$, then $\beta$, and finally $\alpha$.

We find the total time by breaking it down into three different times, one for each parameter:
\begin{equation*}
    t_{\text{ICL}} = T_1 + T_2 + T_3
\end{equation*}

We show that $\gamma$ emerges in $T_1 = \Theta(N)$, $\beta$ emerges after another $T_2=\Theta(N^2)$, and finally $\alpha$ emerges after another $T_3 = O(N^2)$. This gives the total time:
\begin{equation*}
t_{\text{ICL}} = \Theta(N) + \Theta(N^2) + O(N^2) = \Theta(N^2)
\end{equation*}

\subsection{Emergence of $\gamma$ in $T_1 = \Theta(N)$}

We start in the regime $0 \le \alpha, \beta, \gamma < \tfrac{1}{2}$. We show that $\gamma$ is the first to leave this regime at a time $T_1 = O(N)$.

\subsubsection{Dynamics of $\gamma$}

Using $\alpha, \beta < \tfrac{1}{2}$, we get:
\begin{align*}
    s = \exp\!\Bigl(\beta\,\frac{e^{\alpha}}{e^\alpha + 2N - 2}\Bigr) = 1 + O(1/N)
\end{align*}

Using $\gamma < \tfrac{1}{2}$, we get:
\begin{align*}
\frac{\partial \gamma}{\partial t}
\;&=\; 2\,\frac{s}{s + 2N-1} \;-\; 2\gamma\,\frac{s^{2} + 2N - 1}{(s + 2N-1)^{2}} \\
&\ge\; 2\,\frac{s}{s + 2N-1} \;-\; \frac{s^{2} + 2N - 1}{(s + 2N-1)^{2}} \\
&\ge\; 2\,\frac{1 + O(1/N)}{2N + O(1/N)} \;-\; \frac{2N + O(1/N)}{(2N + O(1/N))^{2}} \\
&\ge\; \frac{1 + O(1/N)}{N} \;-\; \frac{2N + O(1/N)}{4N^{2}} \\
&\ge\; \frac{1}{2N} \;+\; O(1/N^2) \\
\end{align*}
\begin{align*}
\frac{\partial \gamma}{\partial t}
\;&=\; 2\,\frac{s}{s + 2N-1} \;-\; 2\gamma\,\frac{s^{2} + 2N - 1}{(s + 2N-1)^{2}} \\
&\le\; 2\,\frac{s}{s + 2N-1} \\
&\le\; 2\,\frac{1 + O(1/N)}{2N + O(1/N)} \\
&\le\; \frac{1}{N} \;+\; O(1/N^2) \\
\end{align*}

This gives us
\begin{align*}
\frac{\partial \gamma}{\partial t}
\;&=\; \Theta(1/N)
\end{align*}

Integrating over time, we obtain:
\begin{align*}
\gamma(T_1) &\;=\; \int_0^{T_1}\frac{\partial \gamma}{\partial t} dt
\;=\; T_1 \, \Theta(1/N)
\end{align*}

Since $\gamma(T_1) = 1/2$, we get that $T_1 = \Theta(N)$.

\subsubsection{Dynamics of $\alpha$ and $\beta$}

We are left to show that the condition $\alpha, \beta < \tfrac{1}{2}$ holds until $T_1$.

\begin{align*}
\frac{\partial \alpha}{\partial t}
\;&=\; \underbrace{\frac{2\,\beta\,(2N-2)\,(2N-1)\,s\,e^{\alpha}}{(e^\alpha + 2N -2)^{2}}}_{O(1)}
  \;\Biggl(\underbrace{\frac{\gamma}{(s + 2N-1)^{2}}}_{O(1/N^2)} - \underbrace{\frac{\gamma^{2}(s-1)}{(s + 2N-1)^{3}}}_{O(1/N^4)} \Biggr)
\\[8pt]
\;&=\; O(1/N^2) \\[20pt]
\frac{\partial \beta}{\partial t}
\;&=\; \underbrace{\frac{2\,(2N-1)\,s\,e^{\alpha}}{e^\alpha + 2N -2}}_{O(1)}
  \;\Biggl(\underbrace{\frac{\gamma}{(s + 2N-1)^{2}}}_{O(1/N^2)} -
  \underbrace{\frac{\gamma^{2}(s-1)}{(s + 2N-1)^{3}}}_{O(1/N^4)} \Biggr)
\\[8pt]
\;&=\; O(1/N^2)
\end{align*}

Integrating over time, we get $\alpha(T_1) = O(1/N)$ and $\beta(T_1) = O(1/N)$. Therefore, for large enough $N$, it is guaranteed that $\alpha$ and $\beta$ will not reach $1/2$ by the time that $\gamma$ does.

\subsubsection{Non-negativity}

For completeness, we also show that parameters are always increasing within this regime, which guarantees that they will never become negative:

\begin{align*}
\frac{\partial \alpha}{\partial t}
\;&=\; \underbrace{\frac{2\,\beta\,(2N-2)\,(2N-1)\,s\,e^{\alpha}\gamma}{(e^\alpha + 2N -2)^{2}}}_{\ge 0}
  \;\Biggl(\underbrace{\frac{1}{(s + 2N-1)^{2}}}_{\substack{\Theta(1/N^2) \\ \ge 0}} - \underbrace{\frac{\gamma(s-1)}{(s + 2N-1)^{3}}}_{O(1/N^4)} \Biggr)
\;\ge\; 0 \\[12pt]
\frac{\partial \beta}{\partial t}
\;&=\; \underbrace{\frac{2\,(2N-1)\,s\,e^{\alpha}\gamma}{e^\alpha + 2N -2}}_{\ge 0}
  \;\Biggl(\underbrace{\frac{1}{(s + 2N-1)^{2}}}_{\substack{\Theta(1/N^2) \\ \ge 0}} -
  \underbrace{\frac{\gamma(s-1)}{(s + 2N-1)^{3}}}_{O(1/N^4)} \Biggr)
\;\ge\; 0 \\[12pt]
\end{align*}

\subsection{Emergence of $\beta$ after $T_2 = \Theta(N^2)$}

We have now entered a new regime where $0 \le \alpha, \beta \le 1/2$ and $1/2 \le \gamma \le 3/2$. We will show that $\beta$ is the first to leave this regime after an additional time $T_2 = \Theta(N^2)$.

\subsubsection{Bounding $\gamma$}

We begin by showing that $\gamma$ remains bounded below $3/2$. We show that $\partial \gamma / \partial t$ would be negative at $\gamma = 3/2$, which implies that $\gamma$ will never go above $3/2$. We use the fact that $s = 1 + O(1/N)$ whenever $\alpha, \beta = O(1)$: 
\begin{align*}
\frac{\partial \gamma}{\partial t}
\;&=\; 2\,\frac{s}{s + 2N-1} \;-\; 2\gamma\,\frac{s^{2} + 2N - 1}{(s + 2N-1)^{2}} \\
\;&=\; 2\,\frac{s}{s + 2N-1} \;-\; 3\frac{s^{2} + 2N - 1}{(s + 2N-1)^{2}} \\
\;&=\; 2\,\frac{1 + O(1/N)}{2N + O(1/N)} \;-\; 3\frac{2N + O(1/N)}{(2N + O(1/N))^{2}} \\
\;&=\; \frac{1 + O(1/N)}{N} \;-\; \frac{3N + O(1/N)}{2N^{2}} \\
\;&=\; -\frac{1}{2N} \;+\; O(1/N^2) \\
\;&<\; 0
\end{align*}

\subsubsection{Dynamics of $\beta$}

Applying the fact that $\gamma = \Theta(1)$ and $s = 1 + O(1/N) = \Theta(1)$ gives us:
\begin{align*}
\frac{\partial \beta}{\partial t}
\;&=\; \underbrace{\frac{2\,(2N-1)\,s\,e^{\alpha}\gamma}{e^\alpha + 2N -2}}_{\Theta(1)}
  \;\Biggl(\underbrace{\frac{1}{(s + 2N-1)^{2}}}_{\Theta(1/N^2)} -
  \underbrace{\frac{\gamma(s-1)}{(s + 2N-1)^{3}}}_{O(1/N^4)} \Biggr)
\;=\; \Theta(1/N^2)
\end{align*}

By integrating, we obtain the value of $\beta$ after $T_2$:
\begin{align*}
    \beta(T_1+T_2) \;=\; \beta(T_1) \;+\;  \int_{T_1}^{T1\,+\,T_2} \frac{\partial \beta}{\partial t}dt \;=\; O(1/N) + T_2 \,\Theta(1/N^2)
\end{align*}

This gives us that $T_2 \,\Theta(1/N^2) = 1/2$, which implies that $T_2 = \Theta(N^2)$.

\subsubsection{Dynamics of $\alpha$}

For completeness, we must establish that $\alpha$ does not become greater than $1/2$ before $\beta$. This comes from the fact that $\beta$ is always increasing at a faster rate than $\alpha$ in this regime:
\begin{align*}
    \frac{\partial \alpha}{\partial t} \;=\; \underbrace{\frac{\beta(2N - 2)}{e^\alpha+2N-2}}_{<\, 1} \; \frac{\partial \beta}{\partial t} 
\end{align*}

\subsection{Emergence of $\alpha$ in $T_3 = O(N^2)$}

We have entered our last regime, which we define using the constraints $0 \le \alpha \le 1/2$, $1/2 \le \gamma \le 3/2$, and $1/2 \le \beta \le 20$.

We know from before that $\gamma$ remains constrained when $\alpha, \beta = \Theta(1)$. We are left to prove that $\alpha$ becomes greater than $1/2$ in a time $T_3 = O(N^2)$ and it does so before $\beta$ becomes greater than the value 20 (chosen arbitrarily to simplify the proofs).

\subsubsection{Dynamics of $\alpha$}

We establish an upper bound on $T_3$ using a lower bound on $\partial \alpha/\partial t$:
\begin{align*}
\frac{\partial \alpha}{\partial t}
\;&=\; 2\,\gamma\beta\,e^{\alpha}\,\underbrace{\frac{(2N-2)\,(2N-1)\,s}{(e^\alpha + 2N -2)^{2}}}_{1+O(1/N)}
  \;\Biggl(\underbrace{\frac{1}{(s + 2N-1)^{2}}}_{1/(4N^2)+O(1/N^3)} - \underbrace{\frac{\gamma(s-1)}{(s + 2N-1)^{3}}}_{O(1/N^4)} \Biggr) \\[8pt]
\;&>\;\frac{1}{8N^2}+O\Bigl( \frac{1}{N^3} \Bigr)
\end{align*}

Integrating over time gives:
\begin{align*}
    \alpha(T_1+T_2+T_3) &\;=\; \alpha(T_1 + T_2) \;+\;  \int_{T_1\,+\,T_2}^{T1\,+\,T_2\,+\,T_3} \frac{\partial \alpha}{\partial t}dt \\[8pt]
    &\;>\; T_3 \Bigg( \frac{1}{8N^2} + O\Big( 1 / N^3 \Big) \Biggl)
\end{align*}

Applying that $\alpha(T_1+T_2+T_3) \;=\; 1/2$ gives us $T_3 < 4N^2 + O(1/N) = O(N^2)$.

\subsubsection{Dynamics of $\beta$}

Finally, we must show that $\beta$ does not reach 20 during $T_3$. We achieve this using an upper bound on $\partial \beta/\partial t$:

\begin{align*}
\frac{\partial \beta}{\partial t}
\;&=\; 2\,e^{\alpha}\,\gamma\;\underbrace{\frac{(2N-1)\,s}{e^\alpha + 2N -2}}_{1+O(1/N)}
  \;\Biggl(\underbrace{\frac{1}{(s + 2N-1)^{2}}}_{1/(4N^2) + O(1/N^3)} -
  \underbrace{\frac{\gamma(s-1)}{(s + 2N-1)^{3}}}_{O(1/N^4)} \Biggr)
\\[8pt]
\;&<\;\frac{3\sqrt{e}}{4N^2} \;+\; O(1/N^3)
\end{align*}

Integrating over time gives:
\begin{align*}
    \beta(T_1+T_2+T_3) &\;=\; \beta(T_1 + T_2) \;+\;  \int_{T_1\,+\,T_2}^{T1\,+\,T_2\,+\,T_3} \frac{\partial \beta}{\partial t}dt \\[8pt]
    &\;<\; \frac{1}{2} \;+\; T_3 \Bigg( \frac{3\sqrt{e}}{4N^2} + O\Big( 1 / N^3 \Big) \Biggl) \\[4pt]
    &\;<\; \frac{1}{2} \;+\; \Bigg( 4N^2 + O\Big( 1 / N \Big) \Biggl) \Bigg( \frac{3\sqrt{e}}{4N^2} + O\Big( 1 / N^3 \Big) \Biggl) \\[4pt]
    &\;<\; \frac{1}{2} \;+\; 3\sqrt{e} + O(1/N) \\
    &\;<\; 5.45 + O(1/N) \\
    &\;<\; 20
\end{align*}

% ==============================================================

\newpage
\section{Training Details and Additional Figures for \cref{sec:full_dynamics}}
\label{app:model_weights}

We confirm our theoretical result by visualizing the weights during standard training with stochastic gradient descent. We use learning rate $\lambda = 1$ and batch size $B = 512$.

We sample $q \sim \text{unif}\{1,N\}$, and we sample items, labels, and positional embeddings from a multivariate Gaussian:
\begin{gather*}
    (\va_{i})_j \sim \mathcal{N}(0, 1), \quad
    (\vb_{i})_j \sim \mathcal{N}(0, 1),\\
    (\vp_{i})_j \sim \mathcal{N}(0, 1),
\end{gather*}
for all $i \in \{1, \ldots,N\}$ and $j \in \{1, \ldots,D\}$.

\foreach \step in {0100, 0200, 0400} {
    \begin{figure}[H]
        \centering
        \includegraphics[width=0.8\linewidth]{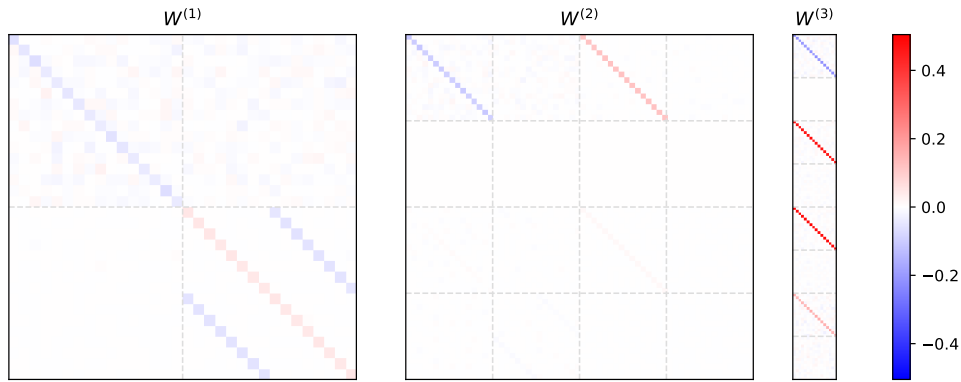}
        \caption{Model weights after $\num{\step}$ training steps with $D = 16$ and $N = 4$.}
        \label{fig:enter-label}
    \end{figure}
}

\newpage
\foreach \step in {0200, 0400, 0800} {
    \begin{figure}[H]
        \centering
        \includegraphics[width=0.8\linewidth]{figures/weights/N=8/\step.pdf}
        \caption{Model weights after $\num{\step}$ training steps with $D = 16$ and $N = 8$.}
        \label{fig:enter-label}
    \end{figure}
}

\newpage
\foreach \step in {0250, 0500, 1000} {
    \begin{figure}[H]
        \centering
        \includegraphics[width=0.8\linewidth]{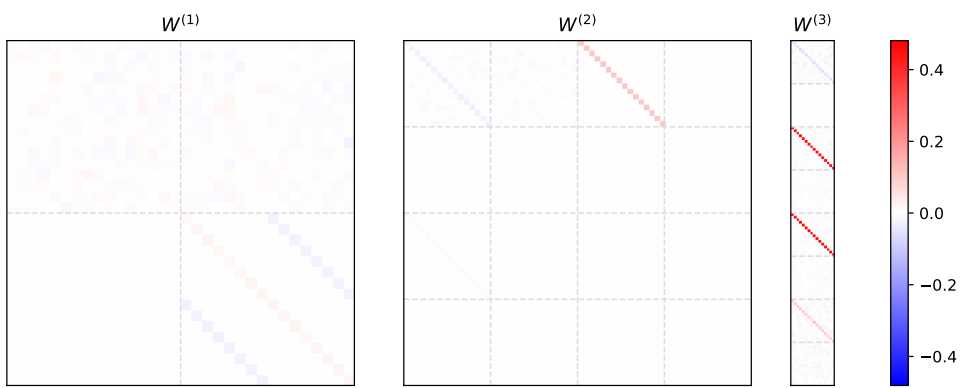}
        \caption{Model weights after $\num{\step}$ training steps with $D = 16$ and $N = 16$.}
        \label{fig:enter-label}
    \end{figure}
}

\newpage
\section{Training Dynamics}
\label{app:dynamics}

As in the main paper, we visualize the pseudo-parameters and loss during standard training, as well as when training only $\alpha_3$, $\beta_2$, and $\gamma_3$. We use $D = 32$, $N =16$, learning rate $\lambda = 1$, and batch size $B=256$. We determine the value of each pseudo-parameter by measuring the magnitude of the parameter vector along the corresponding component.

\begin{figure}[h!]
    \centering
    \includegraphics[width=0.8\linewidth]{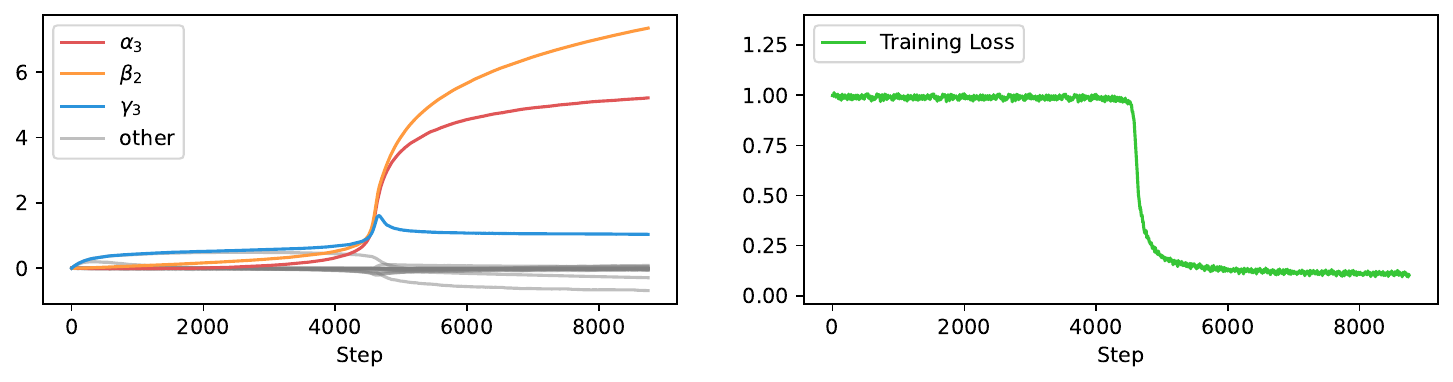}
    \includegraphics[width=0.8\linewidth]{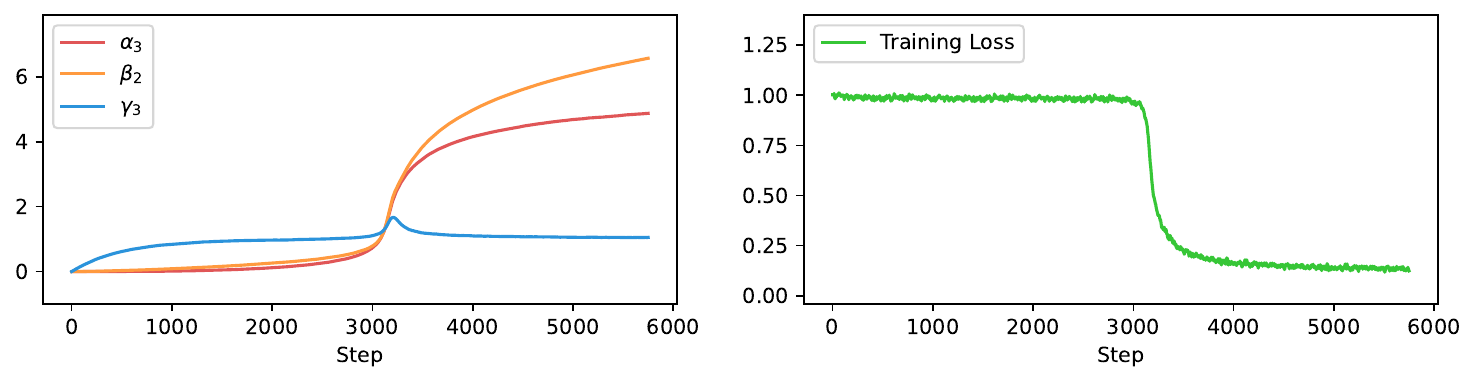}
    \caption{The pseudo-parameters and training loss during training with $D = 16$ and $N = 32$. \\ \emph{Top.} Standard training. \emph{Bottom.} Training only $\alpha_3$, $\beta_2$, and $\gamma_3$.}
    \label{fig:enter-label}
\end{figure}

\begin{figure}[h!]
    \centering
    \includegraphics[width=0.8\linewidth]{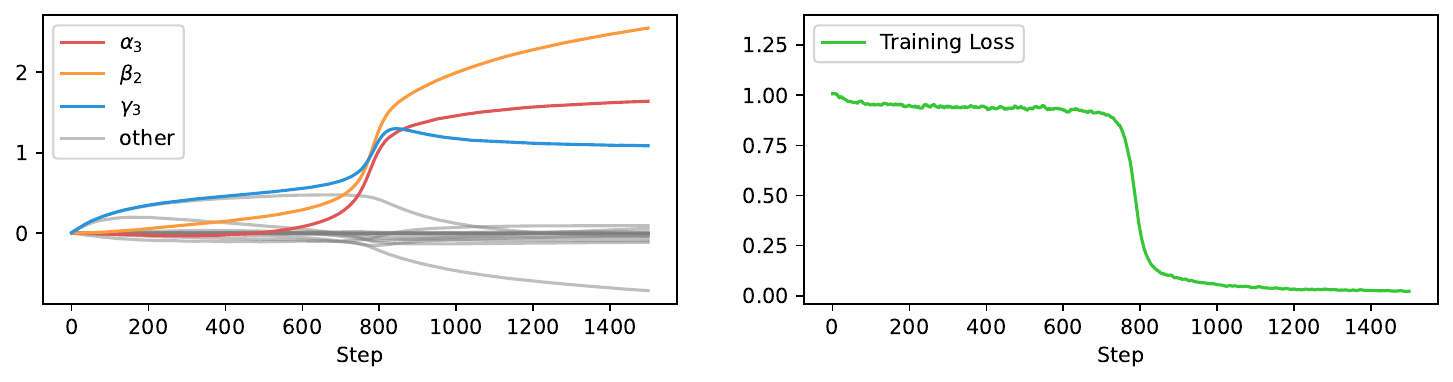}
    \includegraphics[width=0.8\linewidth]{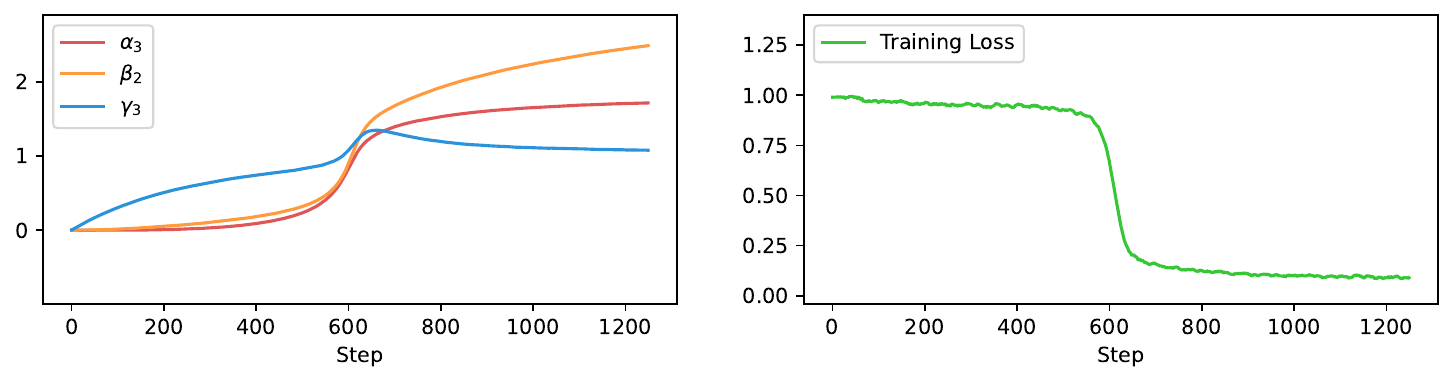}
    \caption{The pseudo-parameters and training loss during training with $D = 32$ and $N = 8$. \\ \emph{Top.} Standard training. \emph{Bottom.} Training only $\alpha_3$, $\beta_2$, and $\gamma_3$.}
    \label{fig:enter-label}
\end{figure}

\newpage
\section{Training Details for Section \ref{sec:standard-transformers}}
\label{app:transformer_training_details}

We train two transformers, a smaller that is used to visualize the weight structure and a larger that is used to visualize the training dynamics.

For visualizing the weight structure in \cref{fig:weights}, we use token and positional embeddings with a vocabulary size of 32, a block size of 32, and an embedding dimension of 2048. Since we have only one head per layer, the head dimension is also 2048. We do not use normalization or weight tying. Following standard practice, we train with AdamW \citep{loshchilov2017decoupled} with learning rate $0.001$ and weight decay $0.01$. We train for 300 steps with 512 sequences per step. Every sequence has length 17 (8 item-label pairs and one query item) and is placed at a random position in the block.  We generate new random sequences for every gradient step as follows: we choose 16 distinct tokens from our vocabulary and group them in item-label pairs; we choose one of the items to be the query; we use the corresponding label as the target output. We construct the input sequences of length 17 by joining the item-label pairs followed by the query item. We use the negative log-likelihood loss applied only at the final position (the query item).

\begin{figure}[h!]
    \centering
    \includegraphics[width=\linewidth]{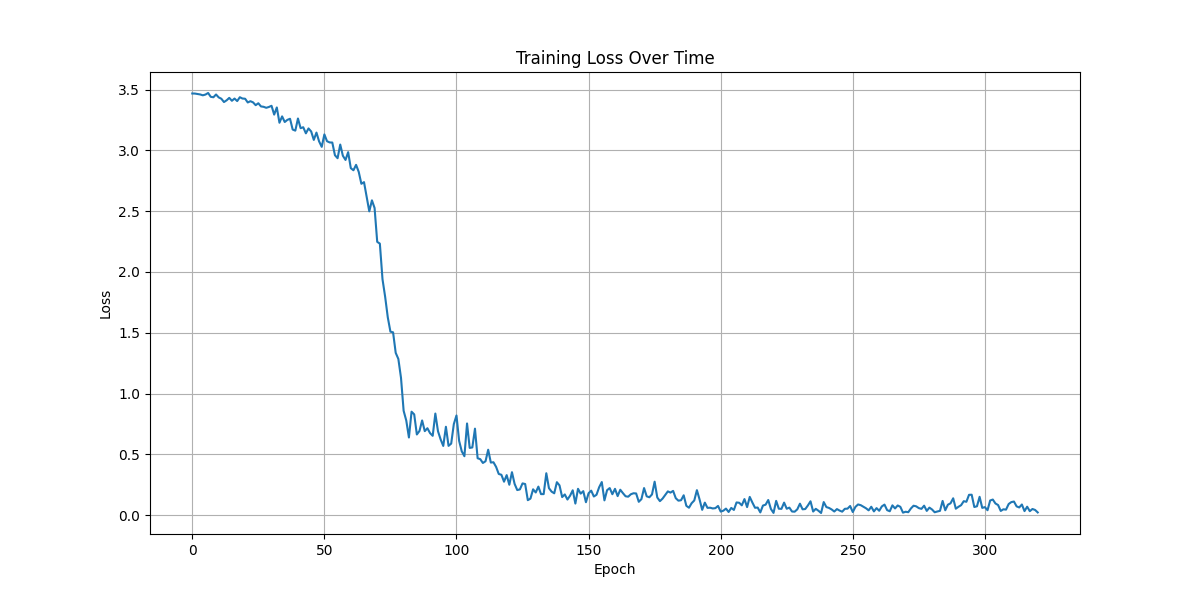}
    \caption{Training loss for the transformer used in \cref{fig:weights}. Note that every batch is generated independently, hence the training loss is also a test loss.}
    \label{fig:enter-label}
\end{figure}

For visualizing the progress measures during training in \cref{fig:progress_measures}, we use token and positional embeddings with a vocabulary size of 64, a block size of 64, and an embedding dimension of 4096. The head dimension is also 4096. We use AdamW with learning rate $0.0001$ and no weight decay, and batch size $4096$. Every sequence has length 33 (16 item-label pairs and one query item).

\newpage
\section{Training Details for Section \ref{sec:induction_head_dynamics}}
\label{app:training_details_ihd}

We empirically validate our theoretical results by measuring the emergence times for different values of $N$. We find that emergence times are in accordance with theoretical predictions. Results are plotted in \cref{fig:emergence}. 
We use $D = 256, \, B=64, \, \lambda = 100$. Following our theoretical assumptions, we use orthonormal inputs, zero initialization, and $q = N$. We constrain the parameters to the 3-dimensional space spanned by $\valpha_3, \vbeta_2,$ and $\vgamma_3$.

\end{document}